\documentclass[a4paper,fleqn]{cas-sc}

\usepackage{amsmath}
\usepackage{amssymb}
\usepackage{amsthm}
\usepackage{mathtools}
\usepackage{bm}
\usepackage{booktabs}
\usepackage[ruled,vlined]{algorithm2e}

\usepackage{float}
\usepackage{subcaption}
\usepackage{makecell}
\usepackage{enumitem}

\usepackage{xurl}
\usepackage[labelfont=,labelsep=period]{caption}
\newtheorem{lemma}{Lemma}
\newtheorem{proposition}{Proposition}

\usepackage{amsmath,amssymb,amsthm} 
\usepackage{geometry} 
\def\tsc#1{\csdef{#1}{\textsc{\lowercase{#1}}\xspace}}
\tsc{WGM}
\tsc{QE}

\ExplSyntaxOn

\cs_set:Npn \__first_footerline:
{
  \group_begin:
  \small
  \sffamily
  \group_end:
}

\ExplSyntaxOff

\begin{document}

\let\WriteBookmarks\relax


 \def\floatpagepagefraction{1}
 \def\textpagefraction{.001}

\shorttitle{NAViLoss}



\title [mode = title]{NAViLoss: An Underwater Navigation-Aware Dual-Residual Objective for Physics-Consistent Learning}


\author[1]{Arup Kumar Sahoo} [
orcid=0000-0003-4515-7434
] \credit{Conceptualization, Methodology, Software, Validation, Investigation, Writing - original draft}
\cormark[1]
\ead{asahoo@campus.haifa.ac.il}

\author[1]{Itzik Klein} [
orcid=0000-0001-7846-0654
] \credit{Supervision, Methodology, Writing - review \& editing}
\ead{kitzik@univ.haifa.ac.il}

\cortext[1]{Corresponding author}

\affiliation[1]{
organization={The Hatter Department of Marine Technologies,
Leon H. Charney School of Marine Sciences,
University of Haifa},
city={Haifa},
postcode={3498838},
country={Israel}
}

\begin{abstract}
Autonomous underwater vehicles (AUVs) commonly rely on inertial navigation systems (INS) aided by Doppler velocity logs (DVLs) for reliable underwater navigation. Accurate DVL velocity estimation is therefore essential for successful operation. Recent learning-based methods have demonstrated improved DVL velocity estimation, particularly under degraded measurement conditions. However, their training objectives typically rely on conventional regression losses that are highly sensitive to large residuals and corrupted observations. Additionally, they do not explicitly account for the physical consistency and measurement uncertainty associated with the underlying sensing process. To address these limitations, this paper introduces navigation-aware loss (NAViLoss), a robust and uncertainty-aware objective function for learning-based AUV velocity estimation. NAViLoss jointly penalizes the velocity-estimation residual in the navigation-state domain and the
beam-consistency residual in the DVL measurement domain. Its bounded
formulation limits the influence of large residuals, while an adaptive
 mechanism regulates the uncertainty in beam geometry.
Furthermore, NAViLoss is integrated with a DeepONet architecture to form a novel NAVi-DeepONet model for seamless estimation of an underwater vehicle's velocity.
Lastly, our model is evaluated
using approximately 10,000m of semi-synthetic
AUV experimental data collected during multiple
real-world sea trials. 
Experimental results demonstrate a 44\% improvement in velocity-estimation accuracy compared with conventional and learning-based baselines. These results demonstrate the effectiveness of navigation-aware and uncertainty-adaptive loss design for robust learning-based underwater velocity estimation.

\end{abstract}




\begin{keywords}
   DeepONet \sep Sensor fusion \sep Loss function \sep Inertial sensors \sep Doppler velocity log  \sep Underwater navigation  \sep   Physics-informed learning \sep Autonomous underwater vehicles
\end{keywords}
\maketitle
\section{Introduction}

Autonomous underwater vehicles (AUVs) have been increasingly deployed for maritime research, seafloor mapping, offshore resource exploration, environmental monitoring, and subsea infrastructure inspection~\cite{ridao2010visual,zhang2022submarine,  tiwari2025multi, tziourtzioumis2026design}. Reliable estimation of their position, velocity, and orientation is fundamental to the successful execution of these operations.
AUVs predominantly operate in underwater environments and have no access to global navigation satellite system (GNSS) signals and external aids. Therefore, they typically rely on inertial navigation systems and Doppler velocity logs (DVL)~\cite{wadoo2017autonomous,zhang2023autonomous, cohen2025adaptive} for accurate navigation. An INS continuously propagates the navigation solution from accelerometer and gyroscope measurements, but measurement noise and
sensor biases lead to accumulated errors and increase drift over time~\cite{groves2013book,titterton2004strapdown,farrell2007gnss}. DVLs are therefore commonly employed to constrain this drift by providing independent velocity information. 
 In bottom-lock mode, four acoustic beams are directed toward
the seabed, and their Doppler frequency shifts are used to determine the
vehicle velocity~\cite{yona2024missbeamnet,damari2026resalignnet}.
The resulting beam measurements are related to the vehicle
velocity through a known DVL observation geometry.

Despite the importance of DVL measurements in underwater navigation, they are susceptible to degradation. Acoustic interference, environmental conditions, vehicle motion, limited bottom visibility, and partial beam availability can introduce corrupted measurements or result in the loss of one or more DVL beams. Such degradation directly affects the quality of the velocity information available for navigation. In particular, missing beams alter the measurement geometry and reduce the information available for estimating the three-dimensional vehicle velocity. Over the years, various model-based approaches have been investigated to improve DVL-based velocity estimation, including least-squares (LS) estimation, loosely and tightly coupled integration, factor-graph optimization, and information-aided estimation methods~\cite{wang2019novel,engelsman2023information,cheng2026robust}. Nevertheless, robust velocity estimation under noisy and incomplete DVL measurements remains an important challenge for reliable AUV navigation.

Recent advances in deep learning, computational power, and open-source libraries have substantially improved the adoption of learning-based (LR-based) technologies to address these limitations.
LR-based methods have been increasingly investigated for underwater velocity estimation, sensor fusion, DVL beam reconstruction, and navigation under degraded measurement conditions~\cite{cohen2022beamsnet,SAHOO2026127766,kang2026hybrid,mo2026hybrid}.
In a pioneer research, DCNet~\cite{yampolsky2025dcnet} employed convolutional neural networks (CNN) to improve DVL calibration, whereas UDON~\cite{zhang2025underwater} applied a data-driven framework for DVL-based underwater odometry. In recent research, a physics-guided  PiDR~\cite{sahoo2026pidr} model was introduced for AUV navigation using inertial measurements. Building on the same concept, A-LUKF~\cite{miao2026physics}, a physics-guided filtering model, was proposed for AUV navigation in degraded underwater environments.
In the domain of multi-sensor fusion, DMIAN~\cite{batovs2026dmian} offered a
LR-based inertial/DVL fusion through navigation error-state
estimation. ResAlignNet~\cite{damari2026resalignnet} further addressed inertial/DVL sensor alignment by estimating the transformation matrix
between the two sensing modalities. These approaches can learn complex relationships between inertial/DVL measurements and have demonstrated improved estimation capability compared with conventional model-based approaches under nominal operating conditions.

While considerable attention has been devoted to developing advanced LR-based architectures, comparatively less attention has been given to the objective functions used to train them. The choice of the loss function influences the optimization behavior, robustness, and generalization capability of the learned model. Existing LR-based approaches are typically optimized
using conventional regression objectives such as mean squared error (MSE), mean absolute error (MAE), Huber loss,
or weighted variants. 
Among these, the MSE remains particularly popular owing to its minimal complexity and computational efficiency. However, its quadratic penalization makes it highly sensitive to
large residuals and corrupted observations. The MAE reduces the influence of large errors through linear penalization, whereas
Huber loss, Smooth-$L_1$, log-cosh, and Charbonnier losses provide intermediate
behavior between quadratic and linear penalties~\cite{wang2022comprehensive, terven2025comprehensive}. More robust
objectives, including the Cauchy, Welsch~\cite{holland1977robust}, Geman--McClure~\cite{chen2024fracgm}, Tukey biweight~\cite{mosteller1977data},
Barron's adaptive robust loss~\cite{barron2019general}, and RoBoSS loss~\cite{akhtar2024roboss}, further reduce the influence
of large residuals through bounded or slowly increasing penalties.
Despite their robustness characteristics, these objectives primarily operate in the prediction
space. They do not explicitly exploit the physical observation model that
relates the predicted state to the measured state.

Physics-informed learning provides a natural means of bridging this gap by incorporating prior physical knowledge into neural-network optimization~\cite{raissi2019physics,chakraverty2025artificial,saini2026pinn}. For DVL-based navigation, the known geometric relationship between the three-dimensional vehicle velocity and the measured beam velocities provides an additional source of physical information during learning. A physics-informed learning objective can therefore consider not only the error between the predicted and reference velocities but also the consistency of the prediction with the observed DVL beams. However, simply combining these residuals using conventional quadratic penalties~\cite{SAHOO2026127766} does not account for changes in measurement uncertainty. 
Additionally, these objectives generally assume complete observations or well-defined physical constraints. 
This limitation becomes particularly important under partial beam availability, where the information provided by the DVL measurements depends on the available beam measurement geometry. Hence, a navigation-oriented learning objective should simultaneously account for prediction accuracy, measurement consistency, robustness to large residuals, and variations in measurement uncertainty.

To address these limitations, this paper introduces \textbf{NAViLoss}, a
navigation-aware robust loss function for LR-based AUV velocity estimation. NAViLoss jointly penalizes residuals in the navigation-state space and the DVL beam-measurement
space.  Its bounded formulation limits the influence of large residuals, while an uncertainty-aware mechanism adapts the contribution of the beam-domain information according to the available DVL measurement geometry.
The proposed objective therefore integrates prediction accuracy, DVL
measurement consistency, robustness, and uncertainty awareness within a
unified optimization framework. In addition to the proposed formulation, we establish the fundamental analytical
properties of NAViLoss and characterize its behavior under both small
and large residual regimes. The proposed loss is further integrated with a DeepONet backbone to form \textbf{NAVi--DeepONet}, enabling robust AUV velocity estimation.

The main contributions of this work are summarized as follows:

\begin{enumerate}

\item A novel navigation-aware objective function, the NAViLoss has been introduced.

\item The theoretical aspects of NAViLoss are proven to hold crucial properties such as: 
vanishing influence for large residuals, quadratic behaviour for small residuals, uncertainty-aware attenuation, and
dual-domain consistency.

\item Fusion of NAViLoss into the uniquely designed DeepONet algorithms to introduce a novel physics-guided NAVi-DeepONet model with two variants.

\item To support the reproduction of results and encourage future research and benchmarking, our codebase has been made publicly available on \url{https://github.com/ansfl/NAViLoss}.
\end{enumerate}

The proposed NAVi-DeepONet framework is evaluated on semi-synthetic AUV data derived
from real-world sea trials conducted in the Mediterranean Sea, Israel,
covering approximately 10,000m of traveled distance. Comprehensive
experiments are performed under noisy sensing, DVL beam outages, and
reduced measurement availability. The experimental results demonstrate
that the proposed framework maintains accurate and robust velocity
estimation across these challenging operating conditions, achieving an
average improvement of 44\% over the considered baseline methods.

The remainder of this paper is organized as follows.
Section~\ref{sec:problem} presents the problem formulation.
Section~\ref{sec:naviloss} introduces the proposed NAViLoss function and
NAVi--DeepONet framework.
Section~\ref{sec:4} discusses the experimental evaluation.
Finally, Section~\ref{sec:conclusion} summarizes the main findings,
discusses the limitations of the proposed approach, and concludes the paper.

\begin{figure*}
    \begin{center}
        \includegraphics[width=\textwidth]{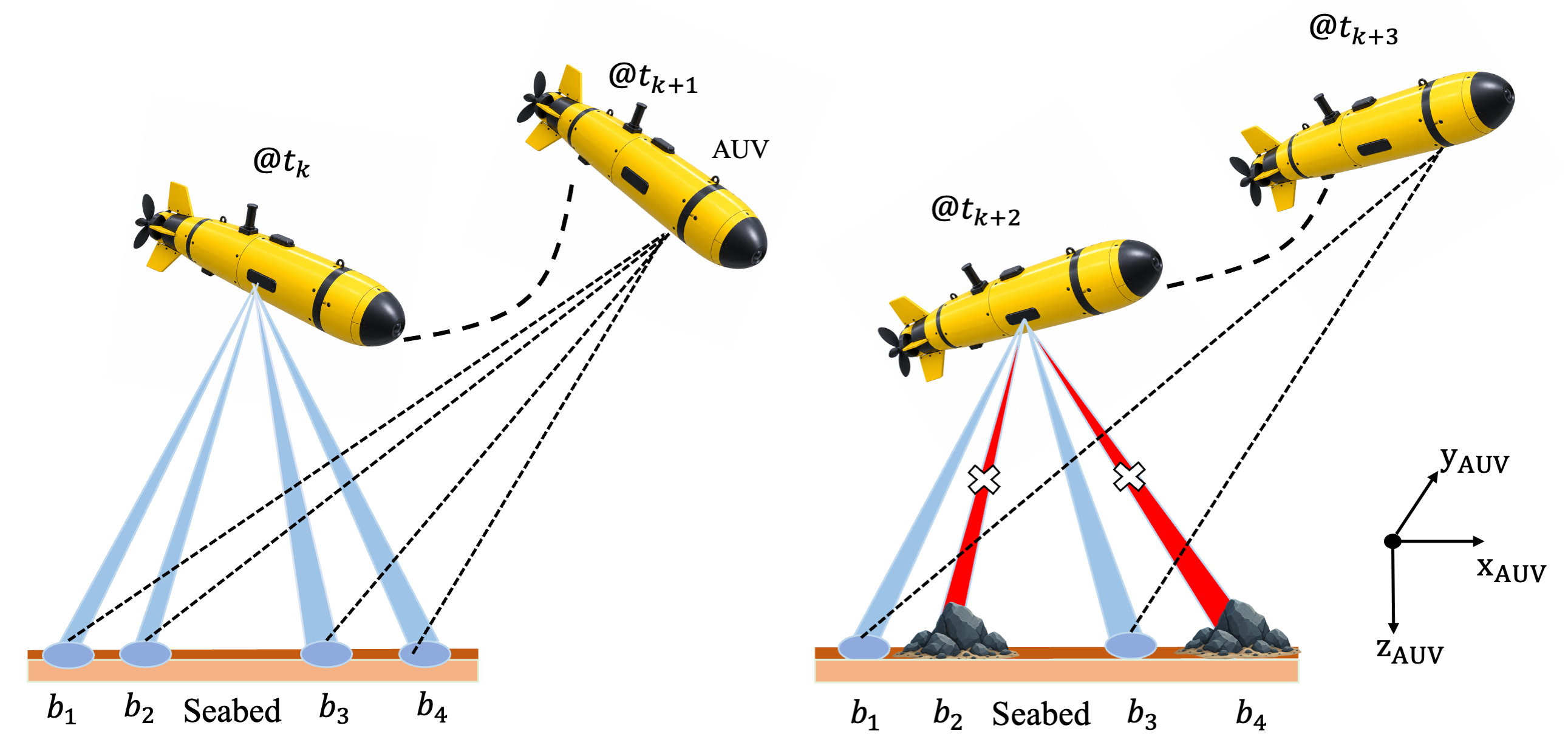}
       \caption{Illustration of an AUV maneuvering with a DVL under full and
partial beam availability. The red beams indicate DVL measurements
obstructed by underwater obstacles, while the remaining beams represent
valid bottom-lock measurements during vehicle motion.}
         \label{fig:1}
    \end{center} 
\end{figure*}


\section{Problem Formulation }
\label{sec:problem}
A DVL consists of four acoustic transducers typically arranged in a Janus (X) configuration. As illustrated in
Figure~\ref{fig:1}, each transducer directs an acoustic beam toward the
seafloor and determines the corresponding Doppler frequency shift from
the reflected signal~\cite{brokloff1994matrix}. Under nominal operating
conditions, all four beams provide valid bottom-lock measurements.
However, underwater obstacles 
may obstruct one or more acoustic beams, resulting in partial DVL beam
availability, as illustrated in Figure~\ref{fig:1}.
Consequently, model-based algorithms may fail or degrade the navigation solutions~\cite{SAHOO2026127766, cohen2024seamless}. 
These challenges motivate the development of robust LR-based approaches to estimate the AUV velocity under different conditions.

Let
\(
\hat{\mathbf v},\mathbf v\in\mathbb{R}^{3}
\)
denote the predicted and reference vehicle velocity vectors, respectively.
The velocity residual is defined as
\begin{equation}
\mathbf r_v=\hat{\mathbf v}-\mathbf v.
\label{eq:velocity_residual}
\end{equation}
Conventional supervised regression primarily minimizes the velocity
residual, typically through an MSE objective such as
\begin{equation}
\mathcal L_{\mathrm{MSE}}
=
\frac{1}{N}
\sum_{i=1}^{N}
\left\|
\mathbf r_{v,i}
\right\|_2^2,
\label{eq:mse1}
\end{equation}

\noindent where $N$ denotes the total number of samples.
Although \eqref{eq:mse1} directly measures velocity-estimation accuracy,
it does not explicitly account for physical consistency between the
predicted velocity and the underlying DVL measurements. Moreover, its
quadratic structure causes large residuals arising from corrupted
measurements or sensor outliers to dominate the optimization process.

Let
\(
\hat{\mathbf b},\mathbf b\in\mathbb{R}^{m}
\)
denote the predicted and measured DVL beam velocities, respectively,
where \(m\leq4\) is the number of available DVL beams. The beam-space
residual is defined as:
\begin{equation}
\mathbf r_b=\hat{\mathbf b}-\mathbf b.
\label{eq:beam_residual}
\end{equation}

\noindent
The predicted beam velocity is obtained from the DVL observation model 
\begin{equation}
\hat{\mathbf b}=H\hat{\mathbf v},
\label{eq:predicted_dvl_model}
\end{equation}
where \(H\in\mathbb{R}^{m\times3}\) denotes the DVL beam geometry matrix.
Thus, the DVL observation model provides an additional physical
relationship between the predicted vehicle velocity and the available
beam measurements. Accordingly, the conventional physics-learning approach can be written as:

\begin{equation} \label{eq:phy_loss}
\mathcal L= \lambda_v
    \left\|\mathbf r_v\right\|_2^2
+
\lambda_b
\left\|\mathbf r_b\right\|_2^2,
\end{equation}

\noindent where $\lambda_v$ and $\lambda_b$ are weight coefficients of velocity and beam residuals, respectively.

Although \eqref{eq:phy_loss} introduces measurement-space consistency, it also remains sensitive to large residuals and assigns a fixed relative importance to the beam-space. However, the physics of the DVL measurements depends on the available beam geometry. In particular, under corrupted or partially available DVL measurements, the confidence associated with the beam-space information may vary considerably from the full-beam configuration. Consequently, simply combining the two residuals as in ~\eqref{eq:phy_loss} does not provide a robust mechanism for balancing prediction accuracy, measurement consistency, and uncertainty.


\section{Proposed Approach} \label{sec:naviloss}

This section presents NAViLoss, the proposed navigation-aware objective function, and NAVi--DeepONet, the corresponding physics-guided learning framework. The following subsections introduce the NAViLoss formulation and establish its analytical properties, followed by the NAVi--DeepONet architecture, DVL operating configurations, and training procedure.

\subsection{Motivation and Overview} 

Conventional regression objectives such as \eqref{eq:mse1} primarily minimize prediction errors without explicitly incorporating the underlying system physics. Although the physics-guided objective in \eqref{eq:phy_loss} introduces DVL measurement consistency, it does not explicitly account for the uncertainty associated with the available navigation measurements. This limitation becomes particularly important under noisy or missing DVL beam conditions, where the information provided by the available DVL geometry can vary considerably.
Accordingly, a navigation-aware objective should
(1) minimize the velocity-estimation residual,
(2) preserve physical consistency with the available DVL measurements,
(3) account for variations in measurement uncertainty,
(4) limit the influence of large residuals, and
(5) remain smooth and continuously differentiable for optimization.

To address these requirements, we propose NAViLoss, a robust and navigation-aware objective that jointly incorporates the velocity and DVL beam-domain residuals. The proposed loss employs a bounded penalty to limit the influence of large residuals and an adaptive mechanism to deal with uncertainty of the beam-domain information according to the available DVL geometry.
Furthermore, NAViLoss is integrated with a DeepONet backbone to form NAVi-DeepONet for AUV velocity estimation.


\subsection{NAViLoss Function}

Based on the velocity and beam residuals defined in
\eqref{eq:velocity_residual} and \eqref{eq:beam_residual},
respectively, the navigation residual space is defined as:
\begin{equation}
\mathcal{R}
=
\mathcal{R}_{v}\times\mathcal{R}_{b},
\label{eq:nav_residual_space}
\end{equation}
where
\(\mathcal{R}_v=\mathbb{R}^{3}\) and
\(\mathcal{R}_b=\mathbb{R}^{m}\)
denote the velocity and beam-residual spaces, respectively.

Let
\(
\sigma_v>0
\)
and
\(
\sigma_b>0
\)
denote the uncertainty measures associated with the velocity residual
\(\mathbf r_v\) and the DVL beam residual \(\mathbf r_b\), respectively.
The corresponding navigation uncertainty vector is defined as
\begin{equation}
\boldsymbol{\sigma}
=
(\sigma_v,\sigma_b),
\end{equation}
with
\begin{equation}
\boldsymbol{\sigma}\in\mathcal U,
\qquad
\mathcal U
=
\mathbb R_{>0}\times\mathbb R_{>0},
\end{equation}

\noindent and $\mathcal U$ is the uncertainty space.
A smaller uncertainty value indicates higher confidence in the corresponding navigation information, whereas a larger value of $\boldsymbol{\sigma}$, represents lower confidence and consequently reduces the contribution of the associated residual during optimization.

Together, the navigation residual space $\mathcal R$ and the uncertainty
space $\mathcal U$ define the input domain of the proposed NAViLoss, expressed as the mapping
\begin{equation}
\mathcal L_{\mathrm{NAVi}}
:
\mathcal R\times\mathcal U
\rightarrow
\mathbb R_{\ge0}.
\label{eq:naviloss_mapping}
\end{equation}

\noindent Specifically, NAViLoss is defined as:

\begin{equation}
\mathcal L_{\mathrm{NAVi}}
=
\mathcal L_v
+
\lambda_b
\mathcal L_b,
\label{eq:naviloss_total}
\end{equation}
where

\begin{equation}
\mathcal L_v
=
\frac{1}{\lambda}
\left(
1-
\frac{1}
{
1+
\lambda
\dfrac{\|\mathbf r_v\|^2}
{\sigma_v^2+\varepsilon}
}
\right),
\label{eq:velocity_naviloss}
\end{equation}
and
\begin{equation}
\mathcal L_b
=
\frac{1}{\lambda}
\left(
1-
\frac{1}
{
1+
\lambda
\dfrac{\|\mathbf r_b\|^2}
{\sigma_b^2+\varepsilon}
}
\right).
\label{eq:beam_naviloss}
\end{equation}

\noindent
Here, $\mathcal L_v$ is the velocity loss, $\mathcal L_b$ is the uncertainty-aware 
beam-consistency physics loss,
$\mathbf r_v\in\mathcal R_v$ is the velocity residual,
$\mathbf r_b\in\mathcal R_b$ is the beam-consistency residual,
$\sigma_v\in\mathcal U_v$ is the velocity uncertainty,
$\sigma_b\in\mathcal U_b$ is the beam uncertainty,
$\lambda>0$ is the robustness parameter, and
$\lambda_b>0$ is the beam-loss weighting coefficient. 
Thus, NAViLoss jointly minimizes the state-estimation error and the measurement-consistency error while accounting for their respective uncertainty levels within a unified navigation-aware optimization framework. The next section proves the theoretical properties of NAViLoss.



\subsection{Theoretical Properties of NAViLoss}

This section analyzes the fundamental mathematical properties of NAViLoss to characterize its behavior with respect to the navigation residuals and their associated uncertainties.

\begin{proposition}[Elementary Properties]
\label{prop:naviloss_properties}

Let $\mathcal{L}_{\mathrm{NAVi}}$ be defined as in
\eqref{eq:naviloss_total}--\eqref{eq:beam_naviloss}, with
$\lambda>0$, $\lambda_b>0$, $\varepsilon>0$,
$\sigma_v>0$, and $\sigma_b>0$.
Then $\mathcal{L}_{\mathrm{NAVi}}$ is non-negative, bounded, and smooth (continuous and infinitely differentiable)
with respect to the residuals $\mathbf r_v$ and $\mathbf r_b$. In
particular,
\begin{equation}
\mathcal{L}_{\mathrm{NAVi}}\in C^\infty
\end{equation}
with respect to $(\mathbf r_v,\mathbf r_b)$. Moreover,
\begin{equation}
0
\leq
\mathcal{L}_{\mathrm{NAVi}}
<
\frac{1+\lambda_b}{\lambda},
\label{eq:naviloss_bound}
\end{equation}
and
\begin{equation}
\lim_{\substack{\|\mathbf r_v\|\rightarrow\infty\\
                 \|\mathbf r_b\|\rightarrow\infty}}
\mathcal{L}_{\mathrm{NAVi}}
=
\frac{1+\lambda_b}{\lambda}.
\label{eq:naviloss_limit}
\end{equation}
\end{proposition}


\begin{lemma}[Bounded Influence of Large Residuals]
\label{lem:robustness}

Let
\begin{equation}
\rho(r)
=
\frac{1}{\lambda}
\left(
1-
\frac{1}
{1+\lambda\dfrac{r^2}{\sigma^2+\varepsilon}}
\right),
\end{equation}
where $\lambda>0$, $\varepsilon>0$, and $\sigma>0$. Then,
\begin{equation}
\lim_{|r|\to\infty}\rho(r)
=
\frac{1}{\lambda},
\qquad
\lim_{|r|\to\infty}\rho'(r)
=
0.
\end{equation}
Hence, the penalty remains bounded while the influence of increasingly
large residuals vanishes. This redescending behavior limits the
contribution of large residuals to the optimization process, providing
robustness against outliers and corrupted measurements. Consequently,
large residuals arising from corrupted DVL measurements or sensor
outliers cannot dominate the training objective.
\end{lemma}


\begin{lemma}[Quadratic Behavior for Small Residuals]
\label{lem:small_residual}

Let
\begin{equation}
\rho(r)
=
\frac{1}{\lambda}
\left(
1-
\frac{1}
{1+\lambda\dfrac{r^2}{\sigma^2+\varepsilon}}
\right),
\end{equation}
where $\lambda>0$, $\sigma>0$, and $\varepsilon>0$. Then,
\begin{equation}
\rho(r)
=
\frac{r^2}{\sigma^2+\varepsilon}
+
O(r^4),
\qquad r\rightarrow0.
\end{equation}
Hence, the NAViLoss exhibits locally quadratic behavior for small residuals.
\end{lemma}

\begin{lemma}[Uncertainty-Aware Attenuation]
\label{lem:uncertainty_attenuation}

For any fixed nonzero residual $r$ and $\sigma>0$, the penalty
$\rho(r;\sigma)$ is strictly decreasing with respect to $\sigma$, i.e.,
\begin{equation}
\frac{\partial \rho(r;\sigma)}{\partial \sigma}<0.
\end{equation}
Hence, residuals associated with larger uncertainty receive a smaller
penalty contribution.
\end{lemma}

\begin{proposition}[Zero-Loss Dual-Domain Consistency]
\label{prop:dual_consistency}

Let
 $\mathcal L_{\mathrm{NAVi}}$ be defined as in
\eqref{eq:naviloss_total}--\eqref{eq:beam_naviloss}, with
$\lambda_b>0$. If
\begin{equation}
\mathcal L_{\mathrm{NAVi}}=0,
\end{equation}
then
\begin{equation}
\widehat{\mathbf v}=\mathbf v,
\qquad
H\widehat{\mathbf v}=\mathbf b.
\end{equation}
Hence, whenever a zero-loss solution exists, it is simultaneously
consistent in the navigation-state and DVL measurement domains.
\end{proposition}

\noindent Readers may referred to~\ref{Apndx} for the proofs of the propositions and lemmas presented above.

The analytical properties established above are further illustrated in
Figures~\ref{fig:loss_comparison} and \ref{fig:uncertainty_effect}.
Figure~\ref{fig:loss_comparison} compares NAViLoss with representative
regression loss functions in terms of their penalty and influence
functions. As established in Lemma~\ref{lem:robustness}, the NAViLoss penalty
remains bounded for large residuals during optimization, while its corresponding influence
approaches zero.
Figure~\ref{fig:uncertainty_effect} illustrates the uncertainty-aware
attenuation property established in Lemma~\ref{lem:uncertainty_attenuation}. Increasing
$\sigma$ attenuates both the penalty and the influence associated with a
given residual, thereby reducing the contribution of uncertain information
during optimization.

\begin{figure*}
    \begin{center}
        \includegraphics[width=\textwidth]{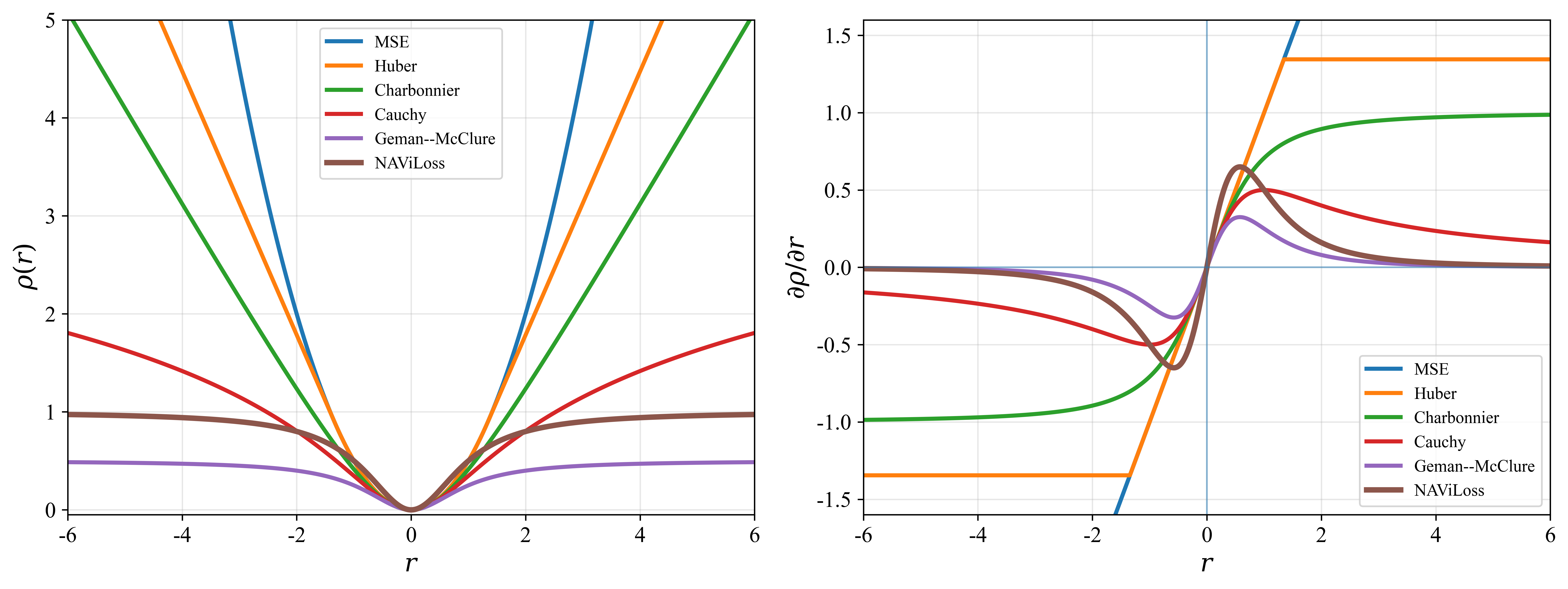}
        \caption{Comparison of NAViLoss with representative regression and robust
loss functions. Left: penalty functions $\rho(r)$ for residual
 $r$. Right: corresponding influence functions. }
         \label{fig:loss_comparison}
    \end{center} 
\end{figure*}

\begin{figure*}
    \begin{center}
        \includegraphics[width=\textwidth]{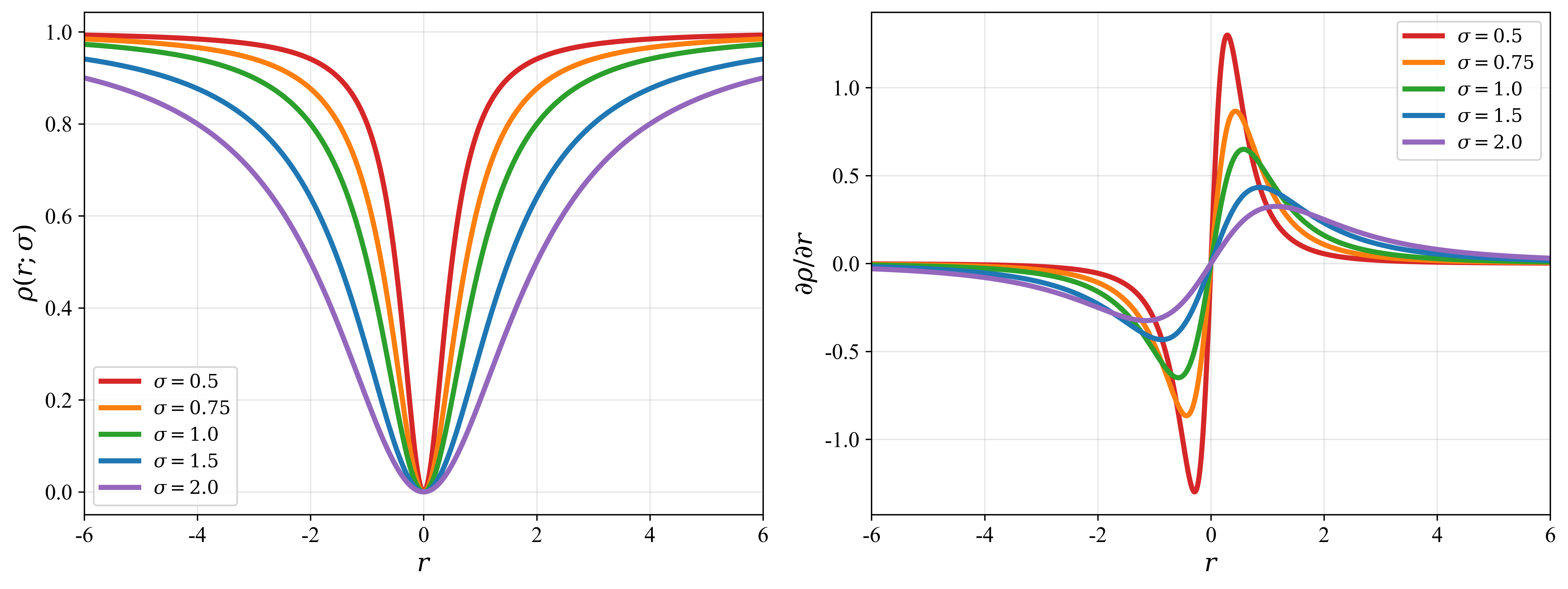}
    \caption{Effect of the uncertainty parameter $\sigma$ on NAViLoss for
fixed $\lambda=1$. Left: penalty function $\rho(r;\sigma)$ for different
uncertainty levels. Right: corresponding influence function.
 }
         \label{fig:uncertainty_effect}
    \end{center} 
\end{figure*}






\subsection{NAVi--DeepONet Architecture}
\label{subsec:deeponet_architecture}

We integrate NAViLoss~\eqref{eq:naviloss_total} with the DeepONet framework~\cite{SAHOO2026127766} and introduce NAVi--DeepONet, a physics-guided learning model for robust AUV velocity estimation.
Figure~\ref{fig:dvl_deeponet} illustrates the proposed architecture, which consists of branch and trunk subnetworks. Since the DVL and IMU operate at different sampling rates, the higher-rate IMU measurements are first temporally aligned with the corresponding DVL epochs. The synchronized four DVL beam measurements and six IMU channels are then arranged over a temporal window of length $W$ to construct the branch input $\mathbf X_k\in\mathbb{R}^{W\times10}$. The branch network extracts features from this temporal input, while the trunk network encodes the corresponding temporal coordinate. Their latent representations are subsequently fused to estimate the three-dimensional AUV velocity. Further details of the synchronization and input construction are provided in Section~\ref{subsec:fullbeam_framework}.

\begin{figure}
    \centering
    \includegraphics[width=1.05\textwidth]{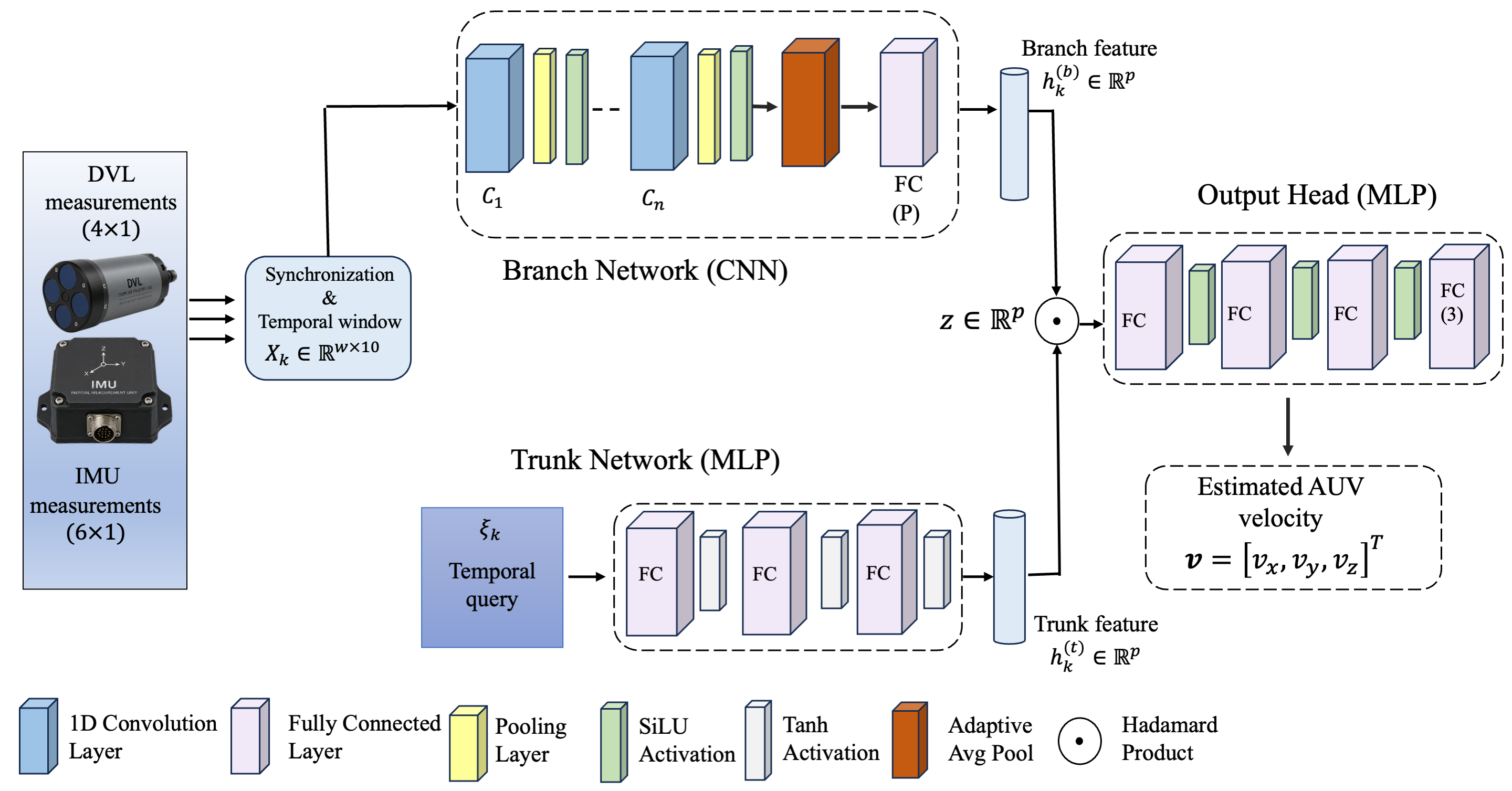}
    \caption{Architecture of the proposed NAVi-DeepONet framework.}
    \label{fig:dvl_deeponet}
\end{figure}
\noindent


Let the temporal sequence of synchronized inertial/DVL measurements is defined as
\begin{equation} \label{eq:23}
\mathbf X_k
=
\left[
\mathbf x_{k-W+1},\ldots,\mathbf x_k
\right],
\end{equation}
where $W$ denotes the temporal window length. 
The branch network maps
$\mathbf X_k$ into a latent representation
\begin{equation}
\mathbf h_k^{(b)}
=
\mathcal B_{\theta_b}(\mathbf X_k)
\in\mathbb R^{p},
\label{eq:new_branch}
\end{equation}
where $\mathcal B_{\theta_b}$ denotes the branch network with trainable
parameters $\theta_b$, and $p$ denotes the latent dimension.


The trunk input is defined by the normalized temporal coordinate
\begin{equation}
\xi_k
=
\frac{t_k-\mu_t}{\sigma_t},
\label{eq:new_time}
\end{equation}
where $\mu_t$ and $\sigma_t$ denote the mean and standard deviation (STD) of
the training timestamps, respectively. The trunk network maps $\xi_k$
into the latent representation
\begin{equation}
\mathbf h_k^{(t)}
=
\mathcal T_{\theta_t}(\xi_k)
\in\mathbb R^{p},
\label{eq:new_trunk}
\end{equation}
where $\mathcal T_{\theta_t}$ denotes the trunk network with trainable
parameters $\theta_t$.

Then the branch and trunk representations are combined through the Hadamard product,
\begin{equation}
\mathbf q_k
=
\mathbf h_k^{(b)}
\odot
\mathbf h_k^{(t)},
\label{eq:new_fusion}
\end{equation}
and the resulting latent feature is passed through the prediction head MLP
\begin{equation}
\hat{\mathbf v}_k
=
\mathcal P_{\theta_p}(\mathbf q_k),
\label{eq:new_head}
\end{equation}
where $\mathcal P_{\theta_p}$ denotes the prediction network. 
Using the branch and trunk representations and their fusion defined in~\eqref{eq:new_branch}--\eqref{eq:new_head}, the complete NAVi--DeepONet approximation is expressed as 
\begin{equation}
\hat{\mathbf v}_k
=
\mathcal F_{\Theta}
(\mathbf X_k,\xi_k)
=
\mathcal P_{\theta_p}
\left[
\mathcal B_{\theta_b}(\mathbf X_k)
\odot
\mathcal T_{\theta_t}(\xi_k)
\right],
\label{eq:new_operator}
\end{equation}
where $\mathcal F_{\Theta}$ denotes the overall NAVi--DeepONet operator,
\(
\Theta=\{\theta_b,\theta_t,\theta_p\}
\)
collects all trainable parameters. During training, NAViLoss jointly incorporates the velocity-estimation and DVL beam-consistency residuals to optimize the network parameters.

\subsection{DVL Operation mode}

Here two sensing configurations are considered
(1) noise-aware velocity estimation using complete DVL beam and IMU measurements and
(2) DVL measurement recovery using partial DVL beam and complete IMU measurements.

\subsubsection{Noise-Aware Velocity Estimation}
\label{subsec:fullbeam_framework}

We first consider the nominal sensing condition in which all four DVL beam
measurements and the complete IMU measurements are available. Let
\begin{equation}
\mathbf b_k
=
[b_{1,k},b_{2,k},b_{3,k},b_{4,k}]^\top
\in\mathbb R^4
\end{equation}
denote the four-beam DVL measurement vector at epoch $k$. The IMU measurement vector is
\begin{equation}
\mathbf u_j
=
[f_{x,j},f_{y,j},f_{z,j},
\omega_{x,j},\omega_{y,j},\omega_{z,j}]^\top
\in\mathbb R^6
\end{equation}
where $f_{x,j}$, $f_{y,j}$, and $f_{z,j}$ are raw specific-force components, and $\omega_{x,i}$, $\omega_{y,i}$, and $\omega_{z,i}$ are raw angular-rate components associated with the $j$-th DVL interval.
A DVL window
of length $W_d$ is constructed as
\begin{equation}
\mathbf D_k
=
[
\mathbf b_{k-W_d+1},
\ldots,
\mathbf b_k
]
\in\mathbb R^{W_d\times4},
\label{eq:full_dvl_window}
\end{equation}
whereas a high-rate IMU window of length $W_i$ is defined as
\begin{equation}
\mathbf I_k
=
[
\mathbf u_{j_k-W_i+1},
\ldots,
\mathbf u_{j_k}
]
\in\mathbb R^{W_i\times6}.
\label{eq:full_imu_window}
\end{equation}

The DVL and IMU histories are encoded independently as
\begin{equation}
\mathbf h_{d,k}
=
\mathcal E_d(\mathbf D_k),
\label{eq:dvl_encoder_full}
\end{equation}
and
\begin{equation}
\mathbf h_{i,k}
=
\mathcal E_i(\mathbf I_k),
\label{eq:imu_encoder_full}
\end{equation}
where $\mathcal E_d(\cdot)$ and $\mathcal E_i(\cdot)$ denote the DVL and
IMU temporal encoders, respectively.
The resulting sensor features,~\eqref{eq:dvl_encoder_full}--\eqref{eq:imu_encoder_full}, are concatenated and projected into the
branch latent space,
\begin{equation}
\mathbf h_k^{(b)}
=
\mathcal C
\left(
[
\mathbf h_{d,k}
\Vert
\mathbf h_{i,k}
]
\right),
\label{eq:full_sensor_fusion}
\end{equation}
where $\Vert$ denotes feature concatenation and $\mathcal C(\cdot)$ denotes
the fusion projection.  It is then combined with the trunk representation as in \eqref{eq:new_fusion} and finally, the resulting latent feature is passed through the prediction head MLP \eqref{eq:new_head} to estimate the AUV velocity.

Using~\eqref{eq:23}, \eqref{eq:new_time}, \eqref{eq:new_operator}, \eqref{eq:full_dvl_window} and \eqref{eq:full_imu_window}, the proposed NAVi--DeepONetV1 operator is written in compact form as
\begin{equation}
\mathcal F_{\Theta}^{\mathrm{full}}
:
(\mathbf D_k,\mathbf I_k,\xi_k)
\mapsto
\hat{\mathbf v}_k.
\label{eq:fullbeam_operator}
\end{equation}

\noindent

\subsubsection{DVL Measurement Recovery}
\label{subsec:missingbeam_framework}

We next consider degraded DVL operation in which one or more DVL beam
measurements are unavailable while the full IMU measurements remain
available. As DVL beams become unavailable, the remaining
beam geometry provides progressively less information about the
three-dimensional vehicle velocity. This setting is therefore well suited
for evaluating the adaptive structure of NAViLoss.


Let
\begin{equation}
\mathbf m_k
=
[m_{1,k},m_{2,k},m_{3,k},m_{4,k}]^\top,
\qquad
m_{i,k}\in\{0,1\},
\label{eq:new_mask_vector}
\end{equation}
denote the DVL beam-availability vector. The corresponding masking matrix is
defined as
\begin{equation}
\mathbf M_k
=
\operatorname{diag}(\mathbf m_k).
\label{eq:new_mask_matrix}
\end{equation}

\noindent
The partially available DVL measurement vector is then expressed as
\begin{equation}
\mathbf b_k^{m}
=
\mathbf M_k\mathbf b_k.
\label{eq:new_partial_beam}
\end{equation}

\noindent
Using a temporal window of length $W_d$, the partial DVL history is
constructed as
\begin{equation}
\mathbf D_k^{m}
=
[
\mathbf b_{k-W_d+1}^{m},
\ldots,
\mathbf b_k^{m}
].
\label{eq:missing_dvl_window}
\end{equation}

\noindent
The full high-rate IMU history $\mathbf I_k$ defined in
\eqref{eq:full_imu_window} remains available in this configuration.
Accordingly, the partial DVL and full IMU histories are encoded as
\begin{equation}
\mathbf h_{d,k}^{m}
=
\mathcal E_d^{m}
\left(
\mathbf D_k^{m}
\right),
\label{eq:missing_dvl_encoder}
\end{equation}
and
\begin{equation}
\mathbf h_{i,k}^{m}
=
\mathcal E_i^{m}
\left(
\mathbf I_k
\right).
\label{eq:missing_imu_encoder}
\end{equation}

\noindent
The two feature representations are fused as
\begin{equation}
\mathbf h_{k,m}^{(b)}
=
\mathcal C_m
\left(
[
\mathbf h_{d,k}^{m}
\Vert
\mathbf h_{i,k}^{m}
]
\right),
\label{eq:missing_sensor_fusion}
\end{equation}
and subsequently combined with the trunk representation,
\begin{equation}
\mathbf q_k^{m}
=
\mathbf h_{k,m}^{(b)}
\odot
\mathbf h_k^{(t)}.
\label{eq:missing_fusion}
\end{equation}

\noindent Finally, the resulting latent feature is passed through the prediction head MLP \eqref{eq:new_head} to estimate the AUV velocity 
\begin{equation}
\hat{\mathbf v}_k
=
\mathcal P_{\theta_p}(\mathbf q_k^{m}).
\label{eq:missing_velocity}
\end{equation}

\noindent Using~\eqref{eq:23}, \eqref{eq:new_time}, \eqref{eq:new_operator}, \eqref{eq:full_dvl_window} and \eqref{eq:missing_dvl_window} the NAVi--DeepONetV2 operator can be compactly written as
\begin{equation}
\mathcal F_{\Theta}^{\mathrm{miss}}
:
(\mathbf D_k^{m},\mathbf I_k,\xi_k)
\mapsto
\hat{\mathbf v}_k.
\label{eq:missing_operator}
\end{equation}

\noindent

\subsection{Training}
\label{subsec:naviloss_training}

For each training sample, let
\(
\mathcal A_k\subseteq\{1,2,3,4\}
\)
denote the set of available DVL beams, and let
\(
H_{\mathcal A_k}
\)
denote the corresponding active-beam geometry submatrix. The uncertainty associated with the beam-space
is determined from the information provided by the active measurement geometry.

The minimum eigenvalue of the active-beam Gram matrix satisfies
\begin{equation}
\lambda_{\min}
\left(H_{\mathcal A_k}^{\top}H_{\mathcal A_k}\right)
=
\min_{\|\mathbf q\|_2=1}
\left\|H_{\mathcal A_k}\mathbf q\right\|_2^2,
\label{eq:min_eigenvalue}
\end{equation}
and therefore measures the weakest-observed direction in the
three-dimensional velocity space. 

Accordingly, for at least three available beams, the beam uncertainty is
defined as
\begin{equation}
\widetilde{\sigma}_{b,k}
=
\frac{1}
{
\max\left(
\lambda_{\min}
\left(
H_{\mathcal A_k}^{\top}H_{\mathcal A_k}
\right),
\varepsilon
\right)
},
\qquad
|\mathcal A_k|\geq3,
\label{eq:beam_uncertainty_geometry}
\end{equation}
where $\varepsilon>0$ prevents numerical instability. Thus, a smaller
minimum eigenvalue produces a larger uncertainty, reducing the influence
of poorly observed beam configurations during training.

When fewer than three beams are available, the corresponding DVL geometry
does not provide sufficient independent information to constrain the
three-dimensional velocity. 
Accordingly, a predefined large uncertainty value is assigned
\begin{equation}
\widetilde{\sigma}_{b,k}
=
\sigma_{\max},
\qquad
|\mathcal A_k|<3,
\label{eq:beam_uncertainty_degenerate}
\end{equation}
where $\sigma_{\max}$ represents high-uncertainty level.

To maintain the uncertainty scale during optimization,
the beam uncertainties are normalized within each mini-batch as
\begin{equation}
\sigma_{b,k}
=
\frac{
\widetilde{\sigma}_{b,k}
}{
\frac{1}{B}
\sum_{j=1}^{B}
\widetilde{\sigma}_{b,j}
+
\varepsilon
}
+
\varepsilon,
\label{eq:beam_uncertainty_normalization}
\end{equation}
where $B$ denotes the mini-batch size.
Consequently, full beam configurations receive relatively
smaller uncertainty values, whereas degraded geometries receive larger
values and therefore contribute less to the beam-space loss.

In this experiment, the velocity uncertainty is fixed as $\sigma_{v,k}=1$. However, it can be adapted when uncertainty information for the reference velocity is available. Such uncertainty is not explicitly modeled in the present experiments. In contrast, the DVL measurement uncertainty varies with the available beam geometry and is therefore adaptively represented through $\sigma_{b,k}$.
The velocity and beam residuals are then evaluated as
\begin{equation} \label{eq:vel_res}
\mathbf r_{v,k}
=
\widehat{\mathbf v}_k-\mathbf v_k,
\end{equation}
and
\begin{equation} \label{eq:beam_res}
\mathbf r_{b,k}
=
H_{\mathcal A_k}\widehat{\mathbf v}_k
-
\mathbf b_{\mathcal A_k,k}.
\end{equation}

The resulting uncertainty-aware residuals are used to compute
$\mathcal L_v$ and $\mathcal L_b$ using \eqref{eq:velocity_naviloss} and \eqref{eq:beam_naviloss}, respectively. Thereby, the
network parameters are updated by minimizing

\begin{equation}
\Theta^\star
=
\arg\min_{\Theta}
\frac{1}{N}
\sum_{k=1}^{N}
\mathcal L_{\mathrm{NAVi}}
\left(
\mathbf r_{v,k},
\mathbf r_{b,k},
\boldsymbol{\sigma}_k
\right),
\label{eq:naviloss_optimization}
\end{equation}
where
\(
\boldsymbol{\sigma}_k
=
(\sigma_{v,k},\sigma_{b,k})
\)
contains the uncertainty measures associated with the velocity and beam
residuals.
Algorithm~\ref{alg:navi_deeponet} summarizes the complete training workflow. 
The proposed DVL-DeepONet framework was trained using the hyperparameters detailed in Table~\ref {tab:model_hyperparameters}.

\begin{algorithm}[!ht]
\caption{Proposed training algorithm for NAVi-DeepONet model.}
\label{alg:navi_deeponet}

Construct inertial/DVL inputs $\mathbf X_k$\;
Construct trunk input $\xi_k$\;
Initialize network parameters $\Theta$\;

\For{each epoch}{
    \For{each mini-batch}{
        Generate the active DVL beam set $\mathcal A_k$\;
        Retain full IMU measurements\;

        Compute branch representation $\mathbf h_k^{(b)}$\;
        Compute trunk representation $\mathbf h_k^{(t)}$\;
        Compute $\mathbf q_k$ using \eqref{eq:new_fusion}\;
        Predict $\hat{\mathbf v}_k$\;

        Compute velocity residual using \eqref{eq:vel_res}\;

        Set $\sigma_{v,k}=1$\;

        Construct active DVL geometry
        $H_{\mathcal A_k}$\;

        \eIf{$|\mathcal A_k|\geq3$}{
            Compute adaptive beam uncertainty using \eqref{eq:beam_uncertainty_geometry}
           
        }{
            Set $\tilde{\sigma}_{b,k}=\sigma_{\max}=10$\;
        }

        Obtain
        $\sigma_{b,k}$ using \eqref{eq:beam_uncertainty_normalization}\;

        Compute beam residual using \eqref{eq:beam_res}\;

        Compute $\mathcal L_v$ using \eqref{eq:velocity_naviloss}\;
        Compute $\mathcal L_b$ using \eqref{eq:beam_naviloss}\;
        Compute total loss using \eqref{eq:naviloss_total}\;

        Update $\Theta$\;

        \If{$\mathcal L_{\mathrm{NAVi}}<\varepsilon$}{
            \Return{$\Theta$}
        }
    }
}

\Return{$\Theta$}
\end{algorithm}

\begin{table*}[htbp]
\centering
\caption{Model architecture and hyperparameters employed for the training of NAVi-DeepONet.}
\label{tab:model_hyperparameters}
\footnotesize
\setlength{\tabcolsep}{7pt}
\renewcommand{\arraystretch}{1.10}
\begin{tabular}{lcc}
\toprule
\textbf{Parameter} &
\textbf{Noise-Aware Velocity Estimation} &
\textbf{DVL Measurement Recovery} \\
\midrule

\multicolumn{3}{l}{} \\

Window length ($W$) 
& 2 & 2 \\

Input features (DVL + IMU) 
& 10 & 10 \\

Number of branch networks 
& 1 & 15 \\
Branch kernel size 
& 3 & 3 \\
Branch CNN channels 
& $10$--$32$--$64$--$64$--$32$
& $10$--$32$--$64$--$64$--$128$ \\

Trunk MLP 
& $1$--$64$--$128$--$32$
& $1$--$64$--$128$--$128$ \\

Prediction head 
& $32$--$128$--$64$--$3$
& $128$--$128$--$64$--$3$ \\

Branch activation 
& SiLU & SiLU \\

Trunk activation 
& Tanh & Tanh \\

Head activation 
& SiLU & SiLU \\

Dropout 
& 0.1 & 0.1 \\

Optimizer 
& AdamW & AdamW \\

Initial learning rate 
& $5\times10^{-4}$ & $10^{-3}$ \\

Weight decay 
& $10^{-5}$ & $10^{-5}$ \\

Batch size 
& 128 & 128 \\

Maximum epochs 
& 100 & 100 \\

Patience 
& 60 & 40 \\

Velocity-loss weight ($\lambda_v$)
& 1.0 & 1.0 \\

Beam-loss weight ($\lambda_b$)
& 1.1 & 0.1 \\

NAViLoss parameter ($\lambda$)
& 0.1 & 0.1 \\

\bottomrule
\end{tabular}
\end{table*}

\section{Analysis and Results} \label{sec:4}
This section discusses the dataset, evaluation metrics, and results obtained by the proposed model. 

\subsection{AUV Dataset}

The experimental data were collected during a series of sea trials in the Mediterranean Sea in Israel, using the University of Haifa's Snapir AUV~\cite{shurin2022autonomous}. Snapir is a modified ECA Robotics A18D mid-size AUV~\cite{eca_a18d_2023}, designed for deep-water autonomus missions at depths of up to 3000m with an operational endurance of approximately 21 hours. The vehicle is equipped with an iXblue Phins Subsea INS~\cite{ixblue_phins_2023} and a Teledyne RDI WorkHorse Navigator DVL~\cite{teledyne_dvl_2023}. The DVL has a nominal velocity STD of 0.02m/s. The INS measurements are recorded at high-rate of 100Hz, while the DVL operates at 1Hz.
During the mission, the vehicle's fail-safe mechanism continuously monitors the DVL quality and
bottom-lock status. If prolonged DVL failure or loss of bottom lock occurs, it aborts the mission and surfaces to ensure safe recovery.

A commonly adopted technique in DVL-learning studies~\cite{cohen2022beamsnet,yampolsky2025dcnet}, is to use the recorded DVL measurements as ground truth (GT) and generate unit under test (UUT) measurements by introducing controlled perturbations.
Following this approach, semi-synthetic UUT measurements are generated through the error-injection model 
\begin{equation}
\label{eq:43}
\tilde{\mathbf y}_k
=
(1+s_{D})\mathbf H\mathbf v_{b,k}^{d}
+
\mathbf b_{D}
+
\sigma_{D}\boldsymbol{\epsilon}_k,
\end{equation}

\noindent
where $\mathbf H\in\mathbb{R}^{4\times3}$ is the DVL beam geometry matrix and $\mathbf v_{b,k}^{d}$ denotes the DVL velocity vector expressed in the body frame,
$s_{D}$ represents the scale-factor error, $\mathbf b_{D}$ denotes the constant beam bias, $\sigma_{D}$ is the STD of the additive noise, and
$\boldsymbol{\epsilon}_k\sim\mathcal{N}(\mathbf 0,\mathbf I_4)$. In particular,  $s_{D}=0.7\%$, $b_{D}=0.001$m/s, and  $\sigma_{D}=0.042$m/s are introduced, consistent with the experimental configuration reported in literature~\cite{cohen2022beamsnet, SAHOO2026127766}.
It reproduce real-world underwater DVL measurement. 
Following noise injection in the beam domain, the measurements are transformed into the DVL coordinate frame using

\begin{equation} \label{eq:1}
\hat{\mathbf v}_b^d
=
(\mathbf H^{T}\mathbf H)^{-1}
\mathbf H^{T}
\tilde{\mathbf y}.
\end{equation}
 and subsequently into the body frame using the rotation matrix $R_d^{b}$. The pipeline for simulating noisy DVL beam measurements from the GT velocity is outlined in Figure~\ref{fig:nosisy}.

\begin{figure*}
    \begin{center}
        \includegraphics[width=\textwidth]{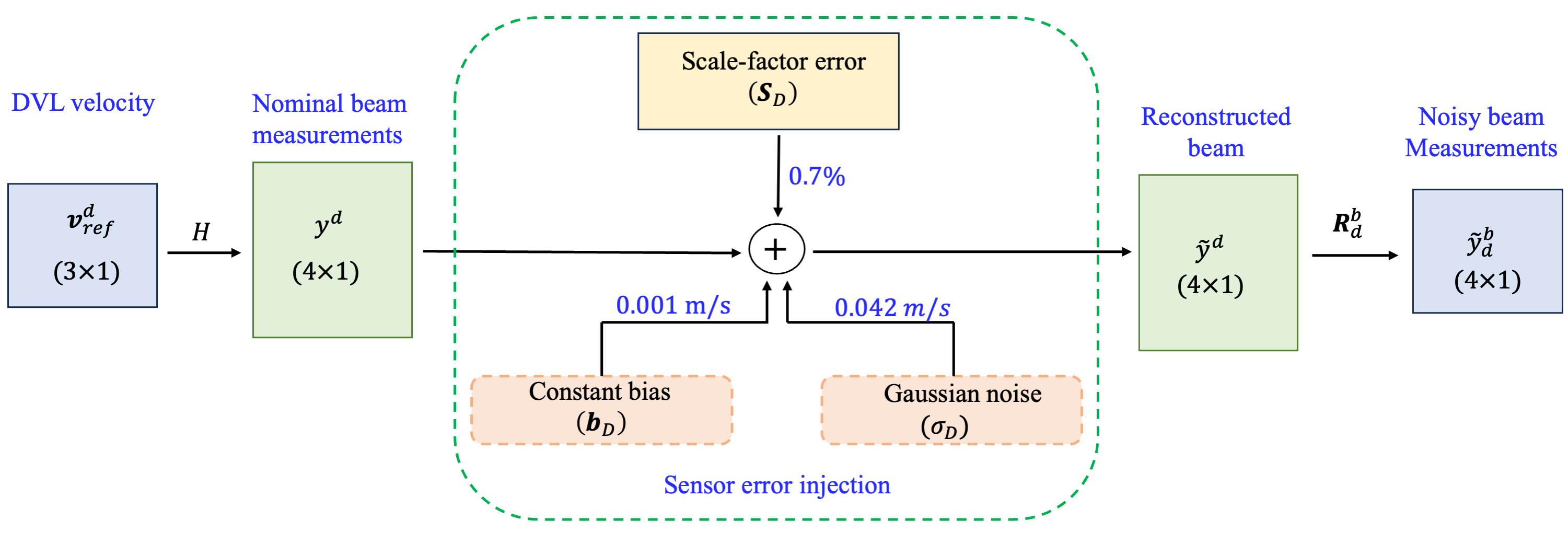}
        \caption{Pipeline for simulating noisy DVL beam measurements from the GT velocity. }
         \label{fig:nosisy}
    \end{center} 
\end{figure*}

The dataset contains thirteen trajectories, denoted T1--T13, totalling 10,000m.
Each trajectory spans approximately 400s, with variations in trajectory geometry, travelled distance, operating depth, and vehicle speed. These variations provide a diverse set of operating conditions for evaluating the proposed framework.
For cross-validation, trajectories are organized into three distinct data folds as shown in Table~\ref{tab:cv_splits}. Unless otherwise specified, the experimental results presented in this work correspond to Fold~1. Figure~\ref{fig:test_trajectories} illustrates the GT paths of the test trajectories across the three folds.

\begin{table*}
\centering
\caption{Three-fold cross-validation scheme used for NAVi-DeepONet evaluation.}
\label{tab:cv_splits}
\footnotesize
\setlength{\tabcolsep}{2pt}
\begin{tabular}{ccccccc}
\toprule
Fold &
Train. Traj. &
Val. Traj. &
Test Traj. &
Train Dist. [m] &
Val. Dist. [m] &
Test Dist. [m] \\
\midrule

F1 &
4,5,6,7,8,9,10,11,12 &
1,13 &
2,3 &
7259.75 &
1496.86 &
1346.59 \\

F2 &
1,6,7,8,9,10,11,12,13 &
2,3 &
4,5 &
7094.26 &

1346.59 &
1566.93 \\

F3 &
1,2,3,4,9,10,11,12,13 &
5,6 &
7,8 &

6683.41 &
1637.56 &
1686.80 \\

\bottomrule
\end{tabular}
\end{table*}

\begin{figure*}
\centering

\subfloat[Trajectory 2]{
\includegraphics[width=0.32\textwidth]{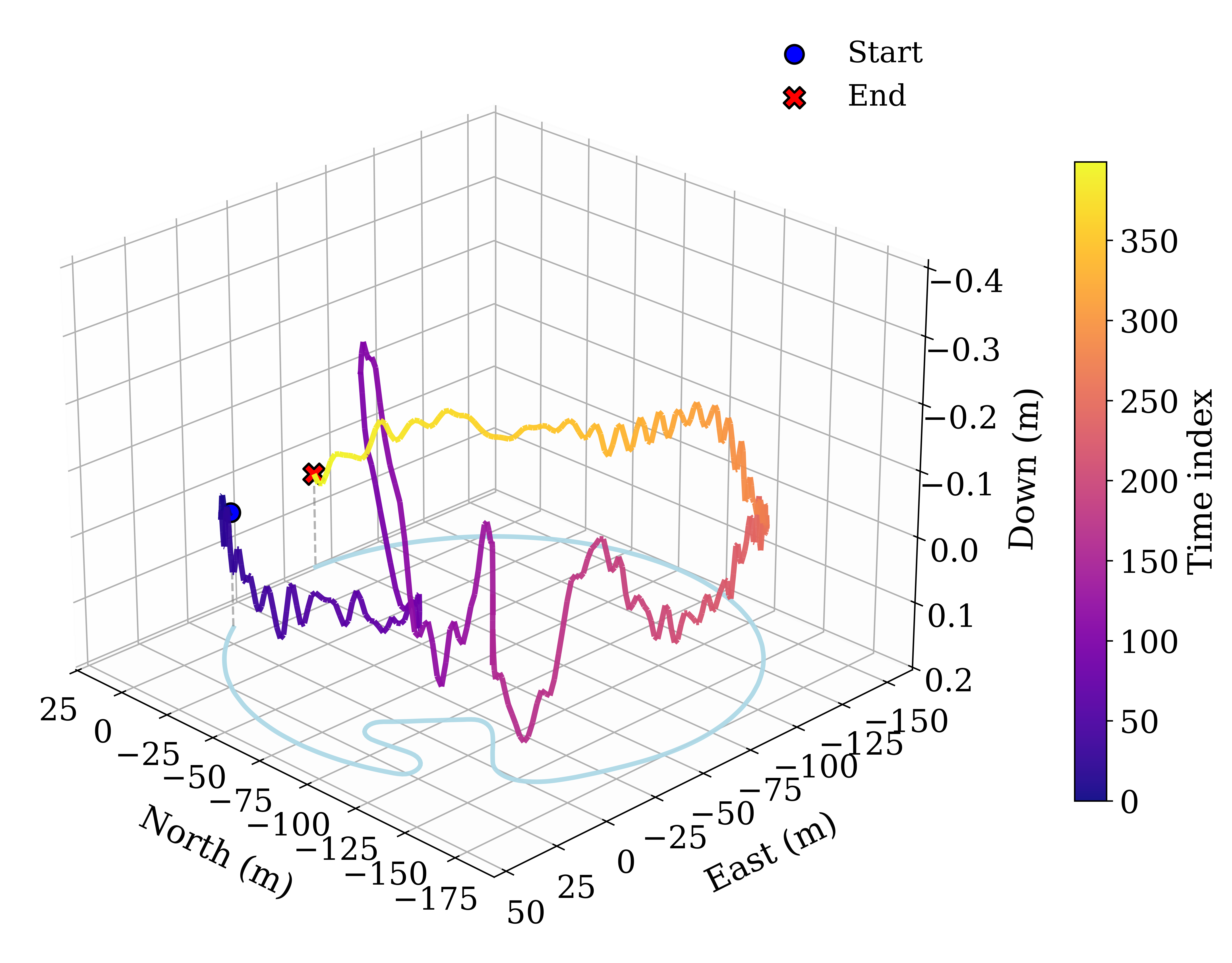}
}
\subfloat[Trajectory 3]{
\includegraphics[width=0.32\textwidth]{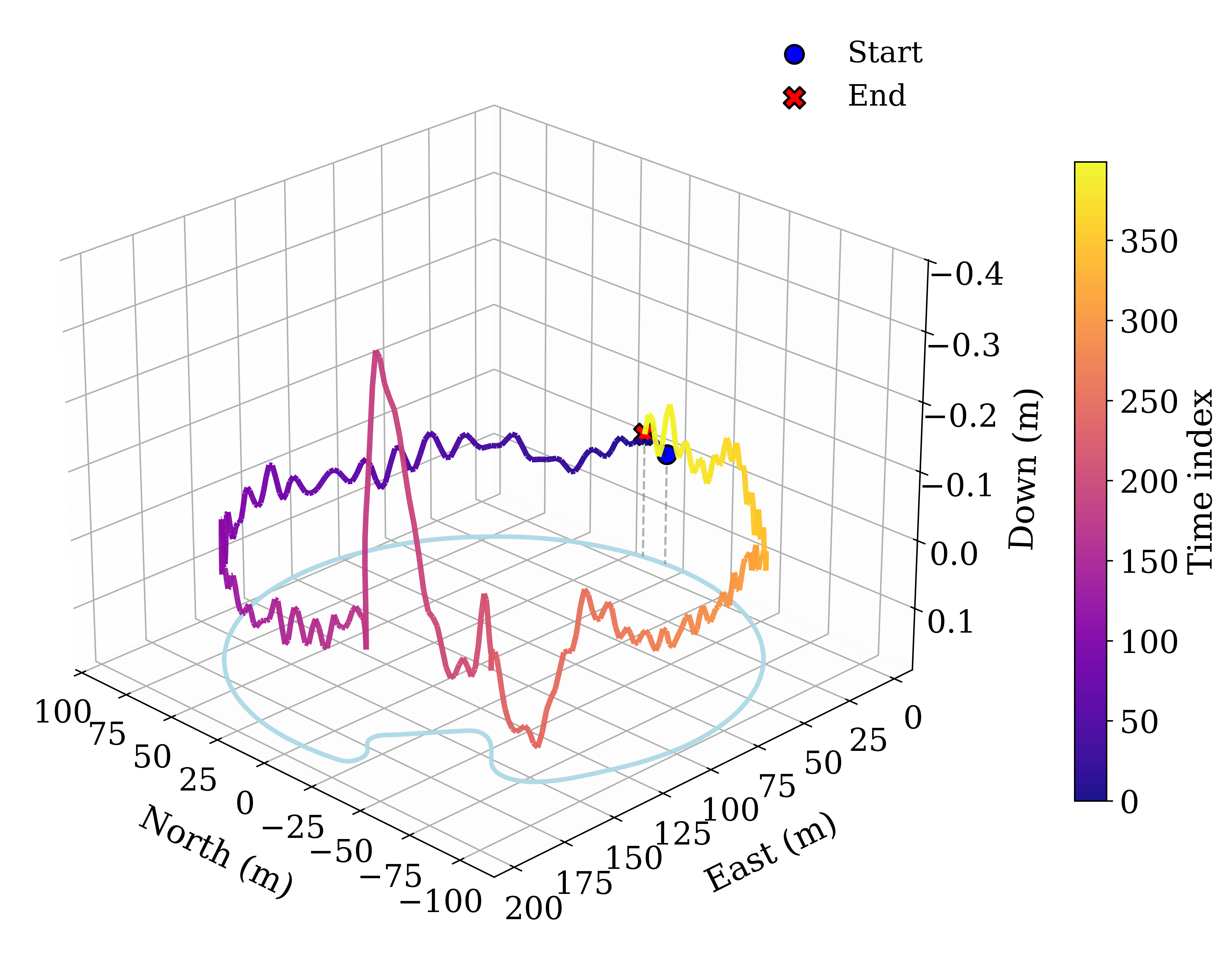}
}
\subfloat[Trajectory 4]{
\includegraphics[width=0.32\textwidth]{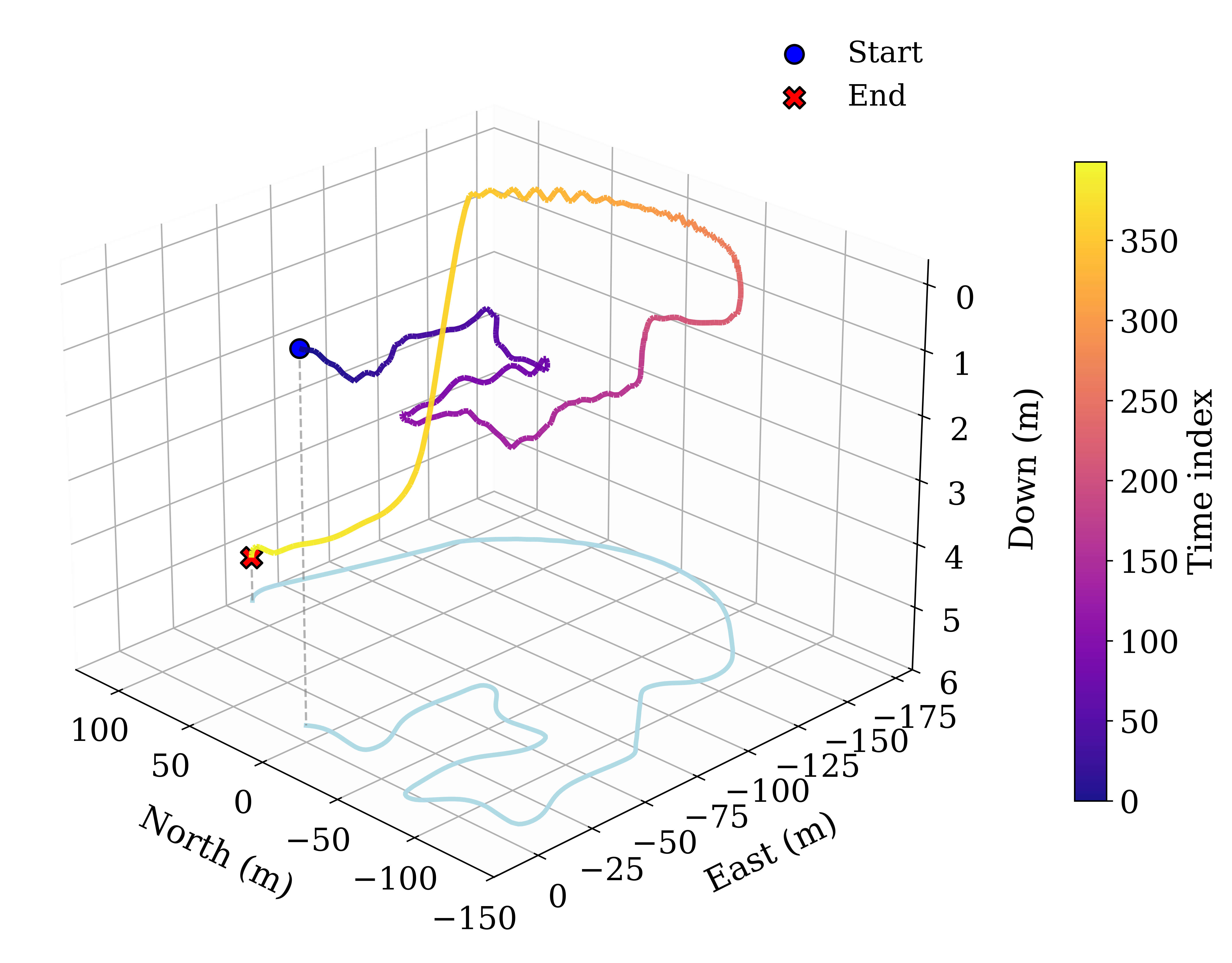}
}

\vspace{0.15cm}

\subfloat[Trajectory 5]{
\includegraphics[width=0.32\textwidth]{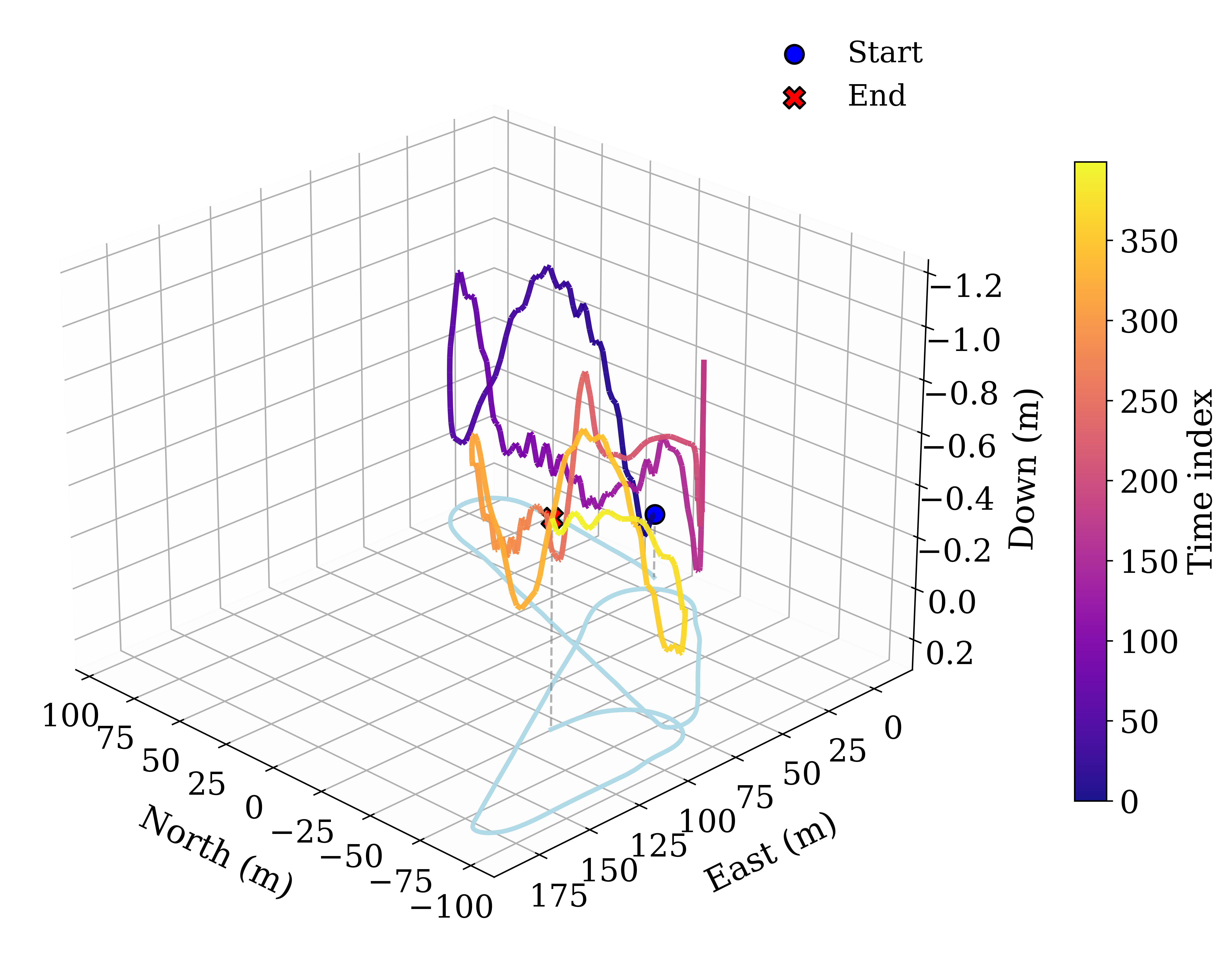}
}
\subfloat[Trajectory 7]{
\includegraphics[width=0.32\textwidth]{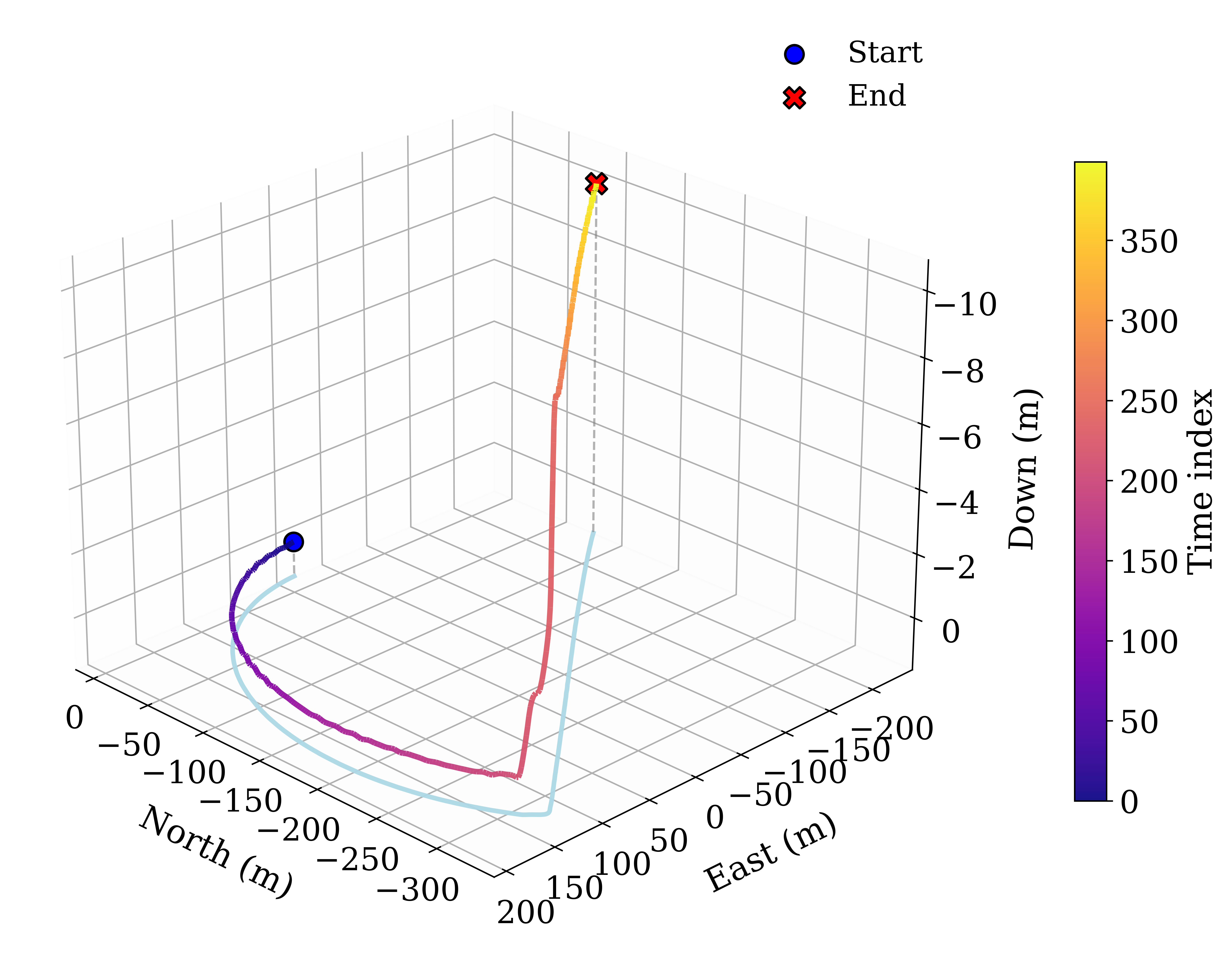}
}
\subfloat[Trajectory 8]{
\includegraphics[width=0.32\textwidth]{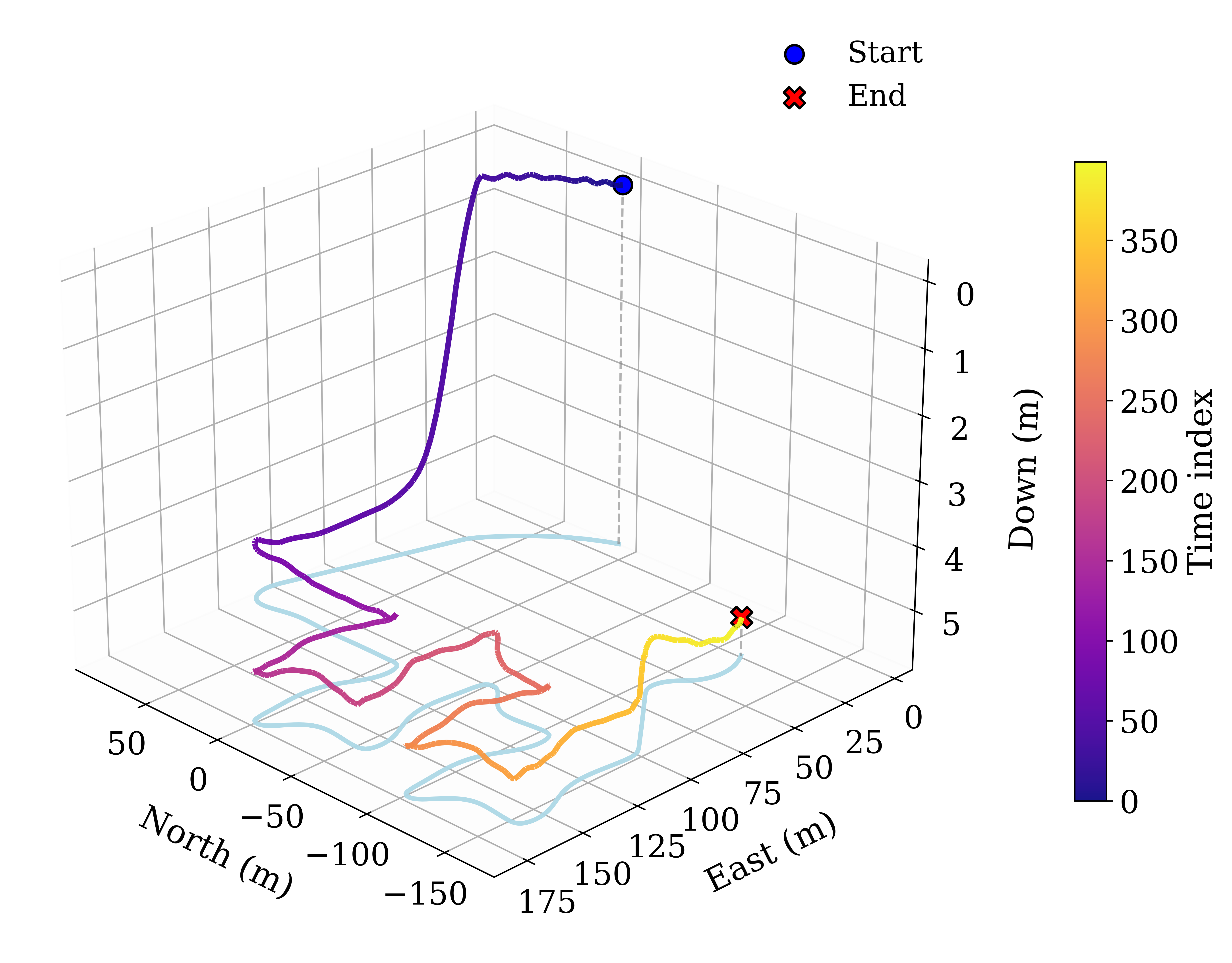}
}
\caption{Three-dimensional GT paths of the test missions, with their corresponding North--East projections shown on the horizontal plane.}

\label{fig:test_trajectories}
\end{figure*}





\subsection{Performance Metrics} \label{metric}

Here, five quantitative metrics have been employed to access the proposed model.
 Let the predicted and GT velocity vectors at time $t_i$ be denoted by $\hat{\mathbf{v}}(t_i)$ and $\mathbf{v}(t_i)$, respectively.

\begin{enumerate}

\item Velocity absolute error (VAE) [m/s]:
\begin{equation}
\text{VAE}_i
=
\left\|
\hat{\mathbf{v}}(t_i)-\mathbf{v}(t_i)
\right\|_2.
\end{equation}

\item Velocity mean absolute error (VMAE) [m/s]:
\begin{equation}
\text{VMAE}
=
\frac{1}{N}
\sum_{i=1}^{N}
\text{VAE}_i,
\end{equation}
where $N$ denotes the total number of samples.

\item Velocity root mean square error (VRMSE) [m/s]:
\begin{equation}
\text{VRMSE}
=
\sqrt{
\frac{1}{N}
\sum_{i=1}^{N}
(\text{VAE}_i)^2
}.
\end{equation}

\item Coefficient of determination ($R^2$):
\begin{equation}
R^2(\dot{x}_j,\hat{\dot{x}}_j)
=
1-
\frac{
\sum_{i=1}^{N}
\left(
\dot{x}_{j,i}-\hat{\dot{x}}_{j,i}
\right)^2
}
{
\sum_{i=1}^{N}
\left(
\dot{x}_{j,i}-\bar{\dot{x}}_j
\right)^2
},
\end{equation}
where $\dot{x}_{j,i}$ and $\hat{\dot{x}}_{j,i}$ denote the GT and predicted values of the $j$-th velocity component, respectively, and $\bar{\dot{x}}_j$ is the mean of the corresponding GT component.

\item Variance accounted for (VAF) [\%]:
\begin{equation}
\mathrm{VAF}(\dot{x}_j,\hat{\dot{x}}_j)
=
\left[
1-
\frac{
\mathrm{var}
\left(
\dot{x}_j-\hat{\dot{x}}_j
\right)
}
{
\mathrm{var}
\left(
\dot{x}_j
\right)
}
\right]
\times 100.
\end{equation}




\end{enumerate}

\subsection{Noise-Aware Velocity Estimation}
\label{sec:noise_resilient}

We first evaluate the effectiveness of proposed NAVi-DeepONetV1 under noisy inertial/DVL
measurements. 
The results are then compared with the model-based
LS estimator, BeamsNet~\cite{cohen2022beamsnet}, and DVL-DeepONet~\cite{SAHOO2026127766} baselines.
Table~\ref{tab:noise_resilient} summarizes the quantitative values of proposed model. 

As expected, the model-based LS estimator is considerably affected by the
measurement perturbations, yielding a VRMSE of 0.178m/s and a VMAE of
0.161m/s. LR-based models, such as BeamsNet and DVL-DeepONet, substantially improve the estimation accuracy,
reducing the VRMSE to 0.116m/s and 0.092m/s, respectively.
However, the best overall performance is obtained by the proposed NAVi-DeepONetV1. It achieves the lowest VRMSE and VMAE of 0.082m/s and
0.071m/s, respectively, while attaining the highest mean $R^2$ and VAF. Our model outperformed the LS, BeamsNet, and DVL-DeepONet baselines by 54\%, 29\%, and 11\% respectively.


These results indicate that the improvement introduced by NAVi-DeepONet is
not limited to a single error metric. The simultaneous reduction in
VRMSE and VMAE, together with the improvements in $R^2$ and VAF,
demonstrates better agreement with the underlying velocity dynamics.

\begin{table*}[!t]
\centering
\caption{Performance comparison for noise-aware velocity
estimation. }
\label{tab:noise_resilient}
\resizebox{\linewidth}{!}{
\begin{tabular}{lccccc}
\hline
\textbf{Model}
& \textbf{VRMSE [m/s]$\downarrow$}
& \textbf{VMAE [m/s]$\downarrow$}
& \textbf{Mean $R^2$$\uparrow$}
& \textbf{VAF [\%]$\uparrow$}
& \textbf{VRMSE Gain [\%]$\uparrow$} \\
\hline
LS
& 0.178 & 0.161 & 0.507 & 61.37 & 54 \\

BeamsNet
& 0.116 & 0.104 & 0.706 & 70.64 & 29 \\

DVL-DeepONet
& 0.092 & 0.074 & 0.869 & 87.21 & 11 \\

\textbf{NAVi-DeepONetV1 (ours)}
& \textbf{0.082}
& \textbf{0.071}
& \textbf{0.878}
& \textbf{88.53}
& -- \\
\hline

\end{tabular}}
\end{table*}

\subsection{DVL Measurement Recovery}
\label{sec:missing_beam}

We next investigate the performance of the proposed model under DVL
beam-outage conditions.
Table~\ref{tab:missing_beam} summarizes the performance under the loss of
two DVL beam measurements. The model-based ELC approach~\cite{cohen2024seamless} exhibits the
largest degradation, with a VRMSE of 1.991m/s and a VMAE of 1.696m/s.
Its negative mean $R^2$ and VAF further indicate poor agreement with the
reference velocity under severe measurement loss. MissBeamNet~\cite{yona2024missbeamnet} notably
improves the estimation performance, reducing the VRMSE to 0.329m/s.
However, its negative mean $R^2$ and VAF indicate limited agreement with
the reference velocity dynamics. DVL-DeepONet~\cite{SAHOO2026127766} provides a further improvement, attaining a VRMSE of
0.254m/s and a VMAE of 0.224m/s.

Furthermore, the proposed NAVi-DeepONetV2 achieves
the best overall performance, reducing VRMSE and VMAE to 0.175m/s and
0.142m/s, respectively. This corresponds to reductions of approximately
31\% in VRMSE and 37\% in VMAE relative to the next-best baseline, DVL-DeepONet. Moreover,
the mean $R^2$ increases from 0.312 to 0.694, while the VAF improves from
48\% to 72\%.

These results demonstrate that the proposed framework remains effective
under reduced DVL measurement availability.

\begin{table*}[!t]
\centering
\caption{Performance comparison under DVL measurement recovery.}
\label{tab:missing_beam}
\resizebox{\linewidth}{!}{
\begin{tabular}{lccccc}
\hline
\textbf{Model}
& \textbf{VRMSE [m/s]$\downarrow$}
& \textbf{VMAE [m/s]$\downarrow$}
& \textbf{Mean $R^2$$\uparrow$}
& \textbf{VAF [\%]$\uparrow$}
& \textbf{VRMSE Gain [\%]$\uparrow$} \\
\hline
ELC
& 1.991 & 1.696 & -134.008 & -1816.87 & 91 \\

MissBeamNet
& 0.329 & 0.288 & -0.403 & -3.36 & 47 \\

DVL-DeepONet
& 0.254 & 0.224 & 0.312 & 48.45 & 31 \\

\textbf{NAVi-DeepONetV2 (ours)}
& \textbf{0.175}
& \textbf{0.142}
& \textbf{0.694}
& \textbf{72.19}
& -- \\
\hline
\end{tabular}}
\end{table*}

\subsection{Ablation Study}
To assess the influence of different parameters and hyperparameters choices within the proposed framework, an ablation study was performed considering three aspects: (1) temporal window length, (2) effect of inertial measurements and (3) three-fold cross-validation.

\subsubsection{Effect of Temporal Window Length}
\label{sec:window_ablation}

We first investigate the influence of the temporal window length $W$ on
the performance of NAVI-DeepONetV2.
The window length determines the amount of historical DVL and inertial
information supplied to the network at each prediction step. A longer
window potentially provides richer temporal context; however, it is not practical to apply in a real-world mission.

Table~\ref{tab:window_ablation} reports the VRMSE obtained for different
window lengths. The best performance is achieved with $W=2$, resulting
in a VRMSE of 0.175m/s. For larger temporal windows, the performance
fluctuates between 0.188 and 0.207m/s, but none of these configurations
outperforms the two-step window.
 Accordingly, $W=2$ is selected for the
remaining experiments. This corresponds to a two-second window fitting to our dataset.

\begin{table}[!t]
\centering
\caption{Effect of temporal window length on NAVI-DeepONetV2.}
\label{tab:window_ablation}
\begin{tabular}{cc}
\hline
\textbf{Window Length }
& \textbf{VRMSE [m/s]$\downarrow$} \\
\hline
\textbf{2}  & \textbf{0.175} \\
3           & 0.182 \\
4           & 0.221 \\
5           & 0.197 \\
10          & 0.188 \\
15          & 0.207 \\
20          & 0.190 \\
\hline
\end{tabular}
\end{table}

\subsubsection{Effect of Inertial Data}

The effect of inertial measurement is evaluated in both noise-aware velocity estimation and DVL measurement-recovery scenarios. We evaluate the performance of our model by removing the inertial data and corresponding parameters from both the training and testing datasets. 
As shown in Table~\ref{tab:noise_resilient_1} for noise-aware velocity estimation, the proposed model achieves a 9\% improvement compared to DVL-only learning.
However, the benefit becomes more pronounced when partial DVL beams are available. As reported in Table~\ref{tab:missing_beam_2}, our model yields
a 38\% improvement. These results indicate that inertial information is
particularly beneficial when the DVL measurement geometry is degraded.

\begin{table*}[!t]
\centering
\caption{Effect of inertial information for noise-aware velocity estimation.}
\label{tab:noise_resilient_1}
\resizebox{\linewidth}{!}{
\begin{tabular}{lccccc}
\hline
\textbf{Model}
& \textbf{VRMSE [m/s]$\downarrow$}
& \textbf{VMAE [m/s]$\downarrow$}
& \textbf{Mean $R^2$$\uparrow$}
& \textbf{VAF [\%]$\uparrow$}
& \textbf{VRMSE Gain [\%]$\uparrow$} \\
\hline

NAVi-DeepONet (DVL only)
& 0.090 & 0.080 & 0.814 & 81.95 & 9 \\

\textbf{NAVi-DeepONetV1 (ours)}
& \textbf{0.082}
& \textbf{0.071}
& \textbf{0.878}
& \textbf{88.53}
& -- \\
\hline

\end{tabular}}
\end{table*}

\begin{table*}[!t]
\centering
\caption{Effect of inertial information for DVL measurement recovery.}
\label{tab:missing_beam_2}
\resizebox{\linewidth}{!}{
\begin{tabular}{lccccc}
\hline
\textbf{Model}
& \textbf{VRMSE [m/s]$\downarrow$}
& \textbf{VMAE [m/s]$\downarrow$}
& \textbf{Mean $R^2$$\uparrow$}
& \textbf{VAF [\%]$\uparrow$}
& \textbf{VRMSE Gain [\%]$\uparrow$} \\
\hline

NAVi-DeepONet (DVL only)
& 0.282 & 0.243 & 0.188 & 41.95 & 38 \\

\textbf{NAVi-DeepONetV2 (ours)}
& \textbf{0.175}
& \textbf{0.142}
& \textbf{0.694}
& \textbf{72.19}
& -- \\
\hline
\end{tabular}}
\end{table*}

\subsubsection{Three-Fold Cross Validation}
\label{sec:cross_validation}
To evaluate the generalization of the proposed
framework, a three-fold cross-validation experiment was conducted using
two independent random seeds (42 and 1234) for each fold, resulting in six
evaluations per method. 
Table~\ref{tab:multi_run} reports the resulting performance in terms of mean $\pm$ STD. 



The proposed NAVi-DeepONetV2 consistently outperforms the
baselines across the repeated evaluations. Compared with MissBeamNet, it
reduces the VRMSE by approximately 29\% and the VMAE by 31\%.
Relative to DVL-DeepONet, reductions of 9\% and 11\% are obtained
in VRMSE and VMAE, respectively. Moreover, compared with its DVL-only
counterpart, incorporating inertial information reduces the VRMSE by
approximately 20\% and the VMAE by 19\%. The proposed method also
achieves the highest mean $R^2$  and VAF, indicating
improved agreement with the reference velocity dynamics across repeated
experiments.
Additionally, the robustness of the proposed model was also evaluated under
increasing DVL beam loss, considering one, two, and three missing beams.
As shown in Figure~\ref{fig:5}, NAVi-DeepONet
consistently achieves a lower mean VRMSE than the baseline across all
three beam-outage conditions. It shows that the performance advantage becomes
more pronounced as the number of missing beams increases.

Overall, the cross-validation results demonstrate that the
performance gains of NAVi-DeepONetV2 are maintained across different
folds and random initializations rather than being restricted to a
particular run.

\begin{table*}
\centering
\caption{Robustness analysis for DVL measurement recovery over
three runs using two random seeds. }
\label{tab:multi_run}
\resizebox{\textwidth}{!}{
\begin{tabular}{lcccccc}
\hline
\textbf{Method}
& \textbf{VRMSE [m/s]$\downarrow$}
& \textbf{MAE [m/s]$\downarrow$}
& \textbf{Median Error [m/s]$\downarrow$}
& \textbf{Maximum Error [m/s]$\downarrow$}
& \textbf{$R^2$$\uparrow$}
& \textbf{VAF [\%]$\uparrow$} \\
\hline

ELC
& $1.951\pm0.130$
& $1.646\pm0.121$
& $0.921\pm0.015$
& 3.858
& -65.527
& -1099.31 \\

MissBeamNet
& $0.329\pm0.012$
& $0.279\pm0.015$
& $0.244\pm0.023$
& 1.414
& 0.110
& 27.18 \\

DVL-DeepONet
& $0.256\pm0.028$
& $0.218\pm0.025$
& $0.190\pm0.027$
& \textbf{1.064}
& 0.521
& 59.79 \\

NAVi-DeepONet (DVL only)
& $0.293\pm0.025$
& $0.239\pm0.012$
& $0.193\pm0.010$
& 1.361
& 0.416
& 49.33 \\

\textbf{NAVi-DeepONetV2 (ours)}
& $\mathbf{0.235\pm0.033}$
& $\mathbf{0.193\pm0.020}$
& $\mathbf{0.165\pm0.015}$
& 1.163
& \textbf{0.679}
& \textbf{68.85} \\
\hline
\end{tabular}}
\end{table*}


\begin{figure}
    \begin{center}
        \includegraphics[width=0.95\columnwidth]{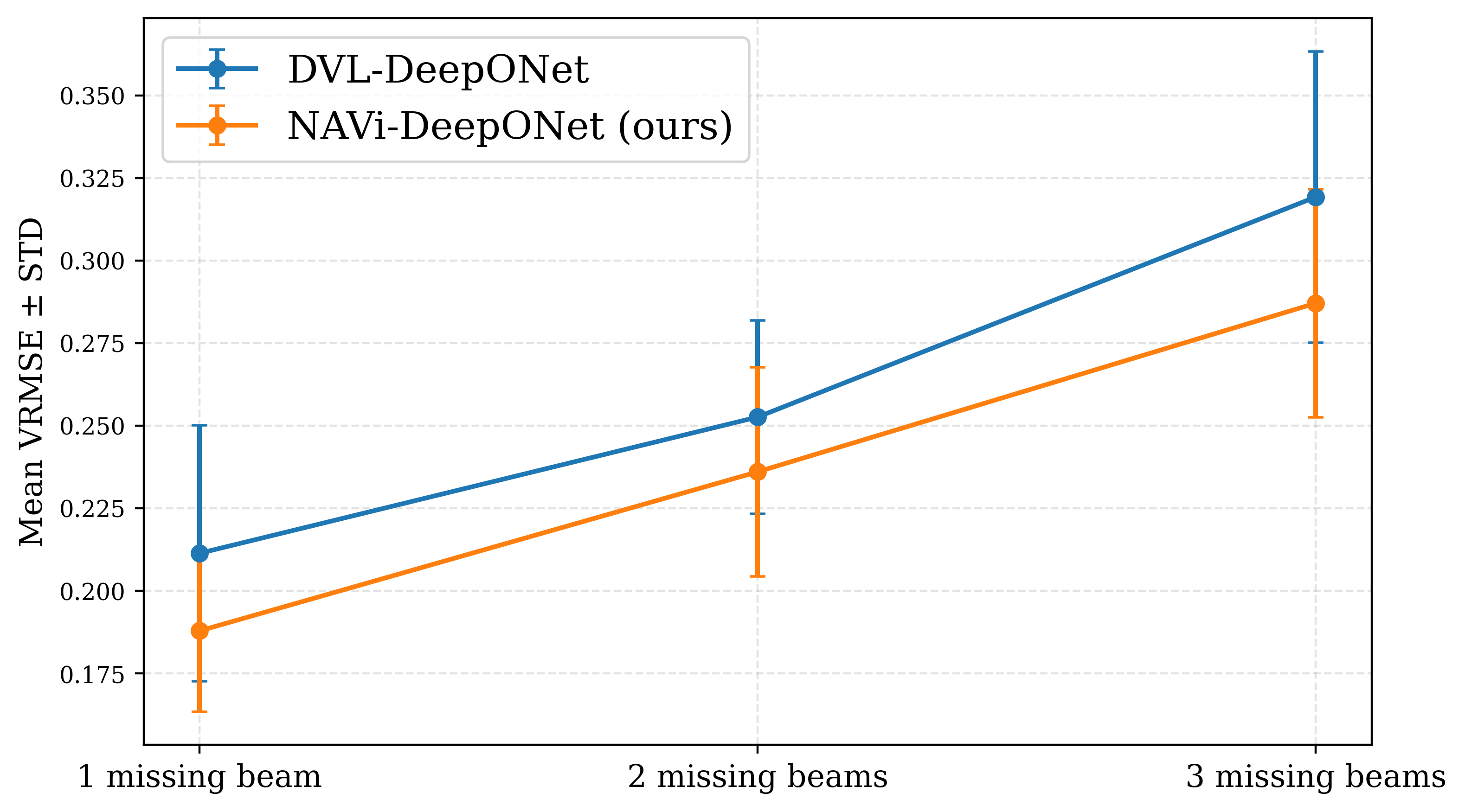}
        \caption{Velocity estimation performance under increasing DVL beam outages.}
         \label{fig:5}
    \end{center} 
\end{figure}







\subsection{Model Complexity}

The computational complexity and inference performance of the proposed NAVi-DeepONet configurations are summarized in Table~\ref{tab:complexity}. NAVi-DeepONetV2 requires approximately 75\% more parameters than NAVi-DeepONetV1 to accommodate different missing-beam scenarios. The additional parameters in NAVi-DeepONetV2 arise from the multiple branches designed to represent different beam-availability patterns. During inference, however, only the branch corresponding to the observed beam pattern is activated. So, the computational cost (MACs and
FLOPs) of NAVi-DeepONetV2 is comparable to that of the NAVi-DeepONetV1 configuration. These results demonstrate the real-time computational feasibility of the proposed framework.

\begin{table*}
\centering
\caption{Computational complexity of the proposed NAVi-DeepONet configurations.}
\label{tab:complexity}
\renewcommand{\arraystretch}{1.}
\resizebox{\textwidth}{!}{
\begin{tabular}{lcccccccc}
\hline
Configuration &
Trainable &
Model Size &
MACs &
FLOPs &
Mean &
P95 &
Throughput &
Real-Time \\
&
Parameters &
[MB] &
[M] &
[M] &
[ms] &
[ms] &
[samples/s] &
Feasible \\
\hline
NAVi-DeepONetV1 &
161,955 & 0.64 & 2.97 & 5.93 & 0.390 & 0.413 & 2564 & Yes \\

NAVi-DeepONetV2 &
675,619 & 2.72 & 2.98 & 5.96 & 0.687 & 0.722 & 14545 & Yes \\
\hline
\end{tabular}}
\end{table*}

\subsection{Summary}
\label{sec:overall_summary}
NAViLoss is incorporated into the DeepONet architecture to construct the proposed NAVi-DeepONet framework. The experimental results collectively demonstrate the effectiveness of proposed model under different scenarios, as shown in Table~\ref{tab:dvl_deeponet_summary}.

\begin{table}
\centering
\caption{Summary of the proposed model with the best VRMSE gain.}
\label{tab:dvl_deeponet_summary}
\footnotesize
\begin{tabular}{lcc}
\toprule

Variant &
\makecell{VRMSE\\{[m/s]} $\downarrow$} &
\makecell{Best Gain\\{[\%]} $\uparrow$} \\
\midrule

NAVi-DeepONetV1
& 0.082
& 54 \\

NAVi-DeepONetV2
& 0.175
& 91 \\

\bottomrule
\end{tabular}
\end{table}

For noise-aware velocity estimation, NAVi-DeepONetV1 achieves the
best overall performance among the evaluated methods. It reduces the VRMSE by
at least 11\% and VMAE by at least 4\%. The contribution of
inertial information is comparatively moderate in this setting, with a
9\% VRMSE and 11\% VMAE improvement over the DVL-only training. This suggests
that when the complete DVL geometry is available, the beam measurements
already provide strong information for three-dimensional velocity
estimation.

The advantage of the proposed framework becomes more pronounced for DVL measurement recovery scenarios, as NAVi-DeepONetV2 achieves at least 31\% and 37\% improvements in VRMSE and VMAE, respectively, relative to baselines. Furthermore, the inclusion of inertial information reduces the VRMSE by approximately
38\% and the VMAE by approximately 42\% compared with the DVL-only training with
NAVi-DeepONet configuration. The repeated cross-validation experiments
support this trend, with the proposed configuration achieves 
approximately 30\% improvement against the counterpart across
different folds and random initializations. The temporal-window ablation
further shows that a short observation window is sufficient, with
$W=2$ providing the best performance under the evaluated missing-beam
conditions.

Overall, the results indicate that the benefit of the proposed framework
depends on the information available from the DVL measurement geometry.
Under full-beam operation, NAViLoss effectively exploits the complete
beam information, whereas under reduced beam availability, its
uncertainty-aware formulation adapts the contribution of the beam-domain
information. These findings demonstrate that NAVi-DeepONet provides a
robust learning framework for accurate velocity estimation
across both complete and degraded DVL measurement conditions.

\section{Conclusion} \label{sec:conclusion}

Reliable velocity estimation is essential for underwater navigation, where DVL measurements play a key role in constraining inertial navigation errors.
For the LR-based models, regression losses primarily minimize the discrepancy between the estimated and GT velocities without explicitly considering the physics of the underlying DVL measurement geometry. This limitation becomes particularly important in the presence of noisy or missing DVL beams, where the uncertainty and amount of available measurement information can vary significantly.
To address these limitations, we introduced NAViLoss, a navigation-aware objective function specifically designed for AUV velocity estimation under both nominal and degraded DVL measurement conditions. Its bounded formulation limits the influence of large residuals, while its uncertainty-aware weighted property adapts the contribution of DVL measurement information according to the available beam geometry. The fundamental properties of NAViLoss were further established through rigorous mathematical analysis and proofs.

Later, NAViLoss is integrated with a DeepONet architecture to construct the proposed
NAVi-DeepONet framework. NAVi-DeepONet is evaluated 
using approximately 10,000m of semi-synthetic
AUV experimental data collected during multiple
real-world sea trials.
Under full-beam measurements,
NAVi-DeepONetV1 achieves a VRMSE of 0.082m/s and improves the accuracy
by an average of 31\% over the baselines. The advantage becomes
more pronounced under missing-beam conditions, where NAVi-DeepONetV2
achieves a VRMSE of 0.175m/s. This corresponds
to an average improvement of 56\% over baselines. Overall, the proposed NAVi-DeepONet model achieves an average improvement of approximately 44\%.

Nonetheless, several limitations warrant consideration. First, the
uncertainty formulation employed in this study is derived from the
available DVL beam geometry and therefore primarily reflects geometric
measurement informativeness rather than a complete probabilistic
characterization of all sensor uncertainties. Second, the robustness
and weighting parameters of NAViLoss are selected empirically for the
dataset considered in this study and may require adjustment for different
DVL configurations, sensor characteristics, or operating environments.
Furthermore, although the proposed framework is evaluated across multiple
data partitions and random initializations, its generalizability across
different AUV platforms, DVL geometries, seabed conditions, and acoustic
environments remains to be investigated. Future work will therefore
consider richer uncertainty
models incorporating sensor statistics, and broader cross-platform and
cross-environment validation.

To conclude, NAViLoss provides a navigation-aware learning objective that
combines (1) robust velocity estimation, (2) DVL measurement consistency, (3)
adaptive uncertainty handling, and (4) limit the influence of large residuals within a unified formulation. When
integrated with NAVi-DeepONet, it enables accurate and resilient velocity
estimation across both nominal and degraded DVL measurement conditions. Our model is useful for LR-based underwater navigation.

\section*{Conflict of Interest Statement}
\noindent The authors confirm that they have no conflicts of interest related to this paper.

\section*{Funding Declaration}
\noindent The authors confirm that they did not receive any funding to carry out this work.

\section*{Data Availability}
\noindent The dataset used in this study is publicly available at: 
\url{https://github.com/ansfl/A-KIT/tree/main}.

\section*{Code Availability}

\noindent The source code associated with this study is publicly available at: \url{https://github.com/ansfl/NAViLoss}.

\appendix

\renewcommand{\thesection}{Appendix \Alph{section}}
\section{}\label{Apndx}

\appendix

\section{Proof of Proposition~\ref{prop:naviloss_properties}}
\label{app:proof_naviloss}

\begin{proof}
Since $\|\mathbf r_v\|^2,\|\mathbf r_b\|^2\geq0$ and the corresponding
uncertainty denominators are strictly positive, it follows directly from
\eqref{eq:velocity_naviloss} and \eqref{eq:beam_naviloss} that
\begin{equation}
0\leq\mathcal L_v<\frac{1}{\lambda},
\qquad
0\leq\mathcal L_b<\frac{1}{\lambda}.
\end{equation}
Therefore,
\begin{equation}
0\leq
\mathcal L_{\mathrm{NAVi}}
<
\frac{1+\lambda_b}{\lambda}.
\end{equation}

Furthermore, since $\sigma_v^2+\varepsilon>0$ and
$\sigma_b^2+\varepsilon>0$, all denominators remain strictly positive
for every finite $\mathbf r_v$ and $\mathbf r_b$. Hence,
$\mathcal L_v$ and $\mathcal L_b$ have no singularities and are
infinitely differentiable with respect to their corresponding residuals.
Consequently, their weighted sum $\mathcal L_{\mathrm{NAVi}}$ is also
smooth with respect to $(\mathbf r_v,\mathbf r_b)$, i.e.,
\begin{equation}
\mathcal L_{\mathrm{NAVi}}\in C^\infty.
\end{equation}

\noindent
Finally, as $\|\mathbf r_v\|\rightarrow\infty$ and
$\|\mathbf r_b\|\rightarrow\infty$,
\begin{equation}
\mathcal L_v\rightarrow\frac{1}{\lambda},
\qquad
\mathcal L_b\rightarrow\frac{1}{\lambda}.
\end{equation}
Thus,
\begin{equation}
\lim_{\substack{\|\mathbf r_v\|\rightarrow\infty\\
                 \|\mathbf r_b\|\rightarrow\infty}}
\mathcal L_{\mathrm{NAVi}}
=
\frac{1+\lambda_b}{\lambda},
\end{equation}
which completes the proof.
\end{proof}

\subsection{Proof of Lemma~\ref{lem:robustness}}
\begin{proof}
Let
\begin{equation}
\rho(r)
=
\frac{1}{\lambda}
\left(
1-
\frac{1}
{1+\lambda\dfrac{r^2}{\sigma^2+\varepsilon}}
\right),
\end{equation}
where $\lambda>0$, $\sigma>0$, and $\varepsilon>0$.

As $|r|\rightarrow\infty$,
\begin{equation}
\lambda\frac{r^2}{\sigma^2+\varepsilon}
\rightarrow\infty,
\end{equation}
and therefore
\begin{equation}
\lim_{|r|\rightarrow\infty}
\frac{1}
{1+\lambda\dfrac{r^2}{\sigma^2+\varepsilon}}
=0.
\end{equation}
Hence,
\begin{equation}
\lim_{|r|\rightarrow\infty}\rho(r)
=
\frac{1}{\lambda}.
\end{equation}

\noindent
Differentiating $\rho(r)$ with respect to $r$ gives
\begin{equation}
\rho'(r)
=
\frac{2r}
{\sigma^2+\varepsilon}
\left(
1+
\lambda\frac{r^2}{\sigma^2+\varepsilon}
\right)^{-2}.
\label{eq:rho_derivative}
\end{equation}

\noindent
For large $|r|$, the denominator grows as $r^4$, whereas the numerator
grows linearly with $r$. Consequently,
\begin{equation}
\lim_{|r|\rightarrow\infty}\rho'(r)=0.
\end{equation}

\noindent Thus, the penalty approaches the finite upper bound $1/\lambda$, while the gradient contribution of the residual to the optimization vanishes for increasingly large residuals. This establishes the bounded and redescending behavior stated in Lemma~\ref{lem:robustness}.

\end{proof}

\subsection{Proof of Lemma~\ref{lem:small_residual}}

\begin{proof}
Let
\begin{equation}
x
=
\lambda\frac{r^2}{\sigma^2+\varepsilon}.
\end{equation}
As $r\rightarrow0$, we have $x\rightarrow0$. Using the expansion
\begin{equation}
\frac{1}{1+x}
=
1-x+x^2+O(x^3),
\qquad x\rightarrow0,
\end{equation}
the penalty can be written as
\begin{align}
\rho(r)
&=
\frac{1}{\lambda}
\left[
1-\left(1-x+x^2+O(x^3)\right)
\right] \\
&=
\frac{1}{\lambda}
\left[
x-x^2+O(x^3)
\right].
\end{align}

\noindent
Substituting
$
x=\lambda r^2/(\sigma^2+\varepsilon)
$
yields
\begin{equation}
\rho(r)
=
\frac{r^2}{\sigma^2+\varepsilon}
-
\lambda
\frac{r^4}{(\sigma^2+\varepsilon)^2}
+
O(r^6).
\end{equation}
Therefore,
\begin{equation}
\rho(r)
=
\frac{r^2}{\sigma^2+\varepsilon}
+
O(r^4),
\qquad r\rightarrow0,
\end{equation}
which establishes the locally quadratic behavior stated in
Lemma~\ref{lem:small_residual}.
\end{proof}

\subsection{Proof of Lemma~\ref{lem:uncertainty_attenuation}}

\begin{proof}
For a fixed $r\neq0$, the penalty is given by
\begin{equation}
\rho(r;\sigma)
=
\frac{1}{\lambda}
\left(
1-
\frac{1}{
1+\lambda\dfrac{r^2}{\sigma^2+\varepsilon}
}
\right).
\end{equation}

\noindent
Differentiating with respect to $\sigma$ yields
\begin{equation}
\frac{\partial \rho(r;\sigma)}{\partial \sigma}
=
-
\frac{
2\sigma r^2
}{
(\sigma^2+\varepsilon)^2
\left(
1+
\lambda\dfrac{r^2}{\sigma^2+\varepsilon}
\right)^2
}.
\label{eq:uncertainty_derivative}
\end{equation}

\noindent
Since $\sigma>0$, $r\neq0$, and $\varepsilon>0$, all terms in the
denominator of \eqref{eq:uncertainty_derivative} are strictly positive,
and $2\sigma r^2>0$. Therefore,
\begin{equation}
\frac{\partial \rho(r;\sigma)}{\partial \sigma}<0.
\end{equation}

\noindent
Hence, $\rho(r;\sigma)$ is strictly decreasing with respect to $\sigma$ for any fixed $r\neq0$. Equivalently, for any $0<\sigma_1<\sigma_2$,
\begin{equation}
\rho(r;\sigma_2)<\rho(r;\sigma_1).
\end{equation}
Thus, increasing the uncertainty parameter attenuates the penalty assigned to a fixed nonzero residual, which proves the result.
\end{proof}

\subsection{Proof of Proposition~\ref{prop:dual_consistency}}

\begin{proof}
From the definition of NAViLoss,
\begin{equation}
\mathcal L_{\mathrm{NAVi}}
=
\mathcal L_v+\lambda_b\mathcal L_b,
\end{equation}
where
\(
\mathcal L_v\geq0
\),
\(
\mathcal L_b\geq0
\), and
\(
\lambda_b>0
\).
Therefore, the condition
\begin{equation}
\mathcal L_{\mathrm{NAVi}}=0
\end{equation}
can hold only if
\begin{equation}
\mathcal L_v=0,
\qquad
\mathcal L_b=0.
\end{equation}

\noindent
From \eqref{eq:velocity_naviloss} and \eqref{eq:beam_naviloss}, each
component loss vanishes if and only if its corresponding residual is
zero. Hence,
\begin{equation} \label{eq:89}
\mathbf r_v=\mathbf 0,
\qquad
\mathbf r_b=\mathbf 0.
\end{equation}

\noindent Using \eqref{eq:velocity_residual}, \eqref{eq:beam_residual}, \eqref{eq:predicted_dvl_model} and \eqref{eq:89} we get
\begin{equation}
\widehat{\mathbf v}-\mathbf v=\mathbf 0,
\qquad
H\widehat{\mathbf v}-\mathbf b=\mathbf 0,
\end{equation}
which yields
\begin{equation}
\widehat{\mathbf v}=\mathbf v,
\qquad
H\widehat{\mathbf v}=\mathbf b.
\end{equation}
This proves the result.
\end{proof}

\vspace{0.2cm}

\bibliographystyle{elsarticle-num}
\bibliography{cas-refs}

\end{document}